\documentclass{article} 
\usepackage{iclr2027_conference,times}

\usepackage{amsmath,amsfonts,bm}

\def\eqref#1{equation~\ref{#1}}

\def\1{\bm{1}}

\DeclareMathAlphabet{\mathsfit}{\encodingdefault}{\sfdefault}{m}{sl}
\SetMathAlphabet{\mathsfit}{bold}{\encodingdefault}{\sfdefault}{bx}{n}

\usepackage{hyperref}
\usepackage{url}

\usepackage{nicematrix, xcolor}
\usepackage{xspace}

\usepackage{booktabs}
\usepackage{multirow}
\usepackage{amsthm}
\usepackage{amssymb}
\usepackage{float}
\usepackage[section]{placeins}
\usepackage{booktabs}
\usepackage{adjustbox}
\usepackage{colortbl}
\usepackage{tikz}
\usepackage{wrapfig}

\definecolor{Gray}{gray}{0.5}
\definecolor{visualfailuregray}{HTML}{BABABA}
\definecolor{utilizationfailurered}{HTML}{D6604D}
\definecolor{failureboundgray}{HTML}{4D4D4D}

\newtheorem{theorem}{Theorem}
\newtheorem{assumption}{Assumption}

\newcommand{\method}{VELens\xspace}
\newcommand{\fullname}{Visual Evidence Lens\xspace}

\title{Encoded but Unused: Bounding the Sources of VLM Errors}

\author{
Anat Veikherman$^{1}$, Mor Ventura$^{1}$, Nitay Calderon$^{2}$,
Lihi Zelnik-Manor$^{1}$ \& Roi Reichart$^{2}$ \\
$^{1}$Faculty of Electrical and Computer Engineering, Technion \\
$^{2}$Faculty of Data and Decision Science, Technion \\
\texttt{\{anatv@campus.technion.ac.il, roiri@technion.ac.il\}}
}

\iclrfinalcopy 
\begin{document}

\maketitle

\begin{abstract}
In this paper, we investigate why vision-language models (VLMs) produce errors by decomposing them into two sources: \emph{visual-representation failures}, in which the visual evidence required for the correct answer is not encoded, and \emph{visual-utilization failures}, in which this evidence is encoded but not successfully used. We propose to estimate the visual-utilization failure rate using lightweight probes and derive theoretical bounds on the true rate that account for probe imperfections. We introduce \method~(\fullname), a rank-based probe that traces visual evidence for both a target object and its ground-truth attribute, such as color, material, or shape, across VLM layers. We evaluate \method on over 6{,}000 images using LLaVA-1.6-7B and Qwen2.5-VL-7B. Our method achieves ROC-AUC scores of 0.95 and 0.91, respectively, and outperforms alternative probes in both low-resource and unseen-attribute settings. Using it to diagnose errors, we estimate that visual-utilization failures account for 84.4\% of LLaVA errors and 70.0\% of Qwen errors, with conservative lower bounds of 77.5\% and 61.0\%, respectively. Overall, our study identifies visual utilization as the dominant bottleneck in VLMs and suggests that a central opportunity for improvement lies in developing better mechanisms to locate, route, and use the visual information they already encode.
\end{abstract}
\section{Introduction}
\label{sec:intro}

Vision-language models (VLMs) increasingly answer questions requiring complex reasoning over images \citep{DBLP:conf/nips/LiuLWL23a,DBLP:conf/cvpr/LiuLLL24,DBLP:conf/cvpr/GanzKABNML24,DBLP:journals/corr/abs-2502-13923}. Yet they still hallucinate, i.e., produce answers inconsistent with the visual content \citep{DBLP:conf/emnlp/LiDZWZW23,DBLP:conf/cvpr/Tong0Z0LX24,DBLP:conf/cvpr/WuX24a,DBLP:journals/corr/abs-2404-18930,DBLP:journals/ijcv/ChenMZYWWS26}. Consider Figure~\ref{fig:TheGoal}: the traffic light is red, but the model answers \emph{green}. For an autonomous vehicle approaching the intersection, such a basic error could determine whether it stops safely. Understanding why it occurs is therefore as important as detecting it: did the error occur because the model never encoded that the traffic light is red, or because it encoded this information but failed to use it when generating the answer?

These possibilities correspond to two distinct failure modes. In a \emph{visual-representation failure}, the internal visual representations lack the evidence required for the correct answer. In a \emph{visual-utilization failure}, that evidence is encoded but not used. The failures are indistinguishable from the final answer, yet point to different limitations and solutions. Representation failures call for improving the visual encoder, whereas utilization failures call for improving how the VLM locates, routes, and incorporates encoded visual evidence during generation. Without separating them, we cannot tell whether VLMs need better visual representations or better utilization.

Existing work on VLM errors primarily defines and detects incorrect outputs \citep{DBLP:conf/emnlp/LiDZWZW23,DBLP:conf/aaai/GunjalYB24,DBLP:journals/corr/abs-2402-00253,DBLP:journals/ijcv/ChenMZYWWS26} 
or mitigates them during training and inference \citep{DBLP:journals/corr/abs-2401-01313,DBLP:conf/acl/WangPDB24,DBLP:conf/cvpr/HuangDZ0H0L0Y24,DBLP:conf/cvpr/LengZCLLMB24,DBLP:conf/iclr/LiuLLWYW24,DBLP:conf/cvpr/YuYZHHCHL0024,DBLP:conf/icml/ZhaoDZG25}. 
Complementary studies examine VLM internals using hidden-representation probes \citep{DBLP:conf/iclr/AlainB17,DBLP:journals/corr/abs-2506-08008,DBLP:conf/eacl/KogilathotaGSZ26}, 
vocabulary-space projections \citep{DBLP:conf/iclr/NeoO0G0B25,DBLP:conf/iclr/JiangKPG25,DBLP:conf/cvpr/JiangCZLSY25}, and cross-modal attention and 
information-flow analyses \citep{DBLP:conf/cvpr/KaduriBD25,DBLP:conf/eccv/LiuZC24,DBLP:conf/cvpr/0009YHS25}. These approaches reveal where visual information appears and how it influences generation, 
but not how often errors reflect missing rather than unused evidence. Even recent multi-stage accounts focus on utilization without showing whether the required visual evidence 
was encoded \citep{DBLP:journals/corr/abs-2603-15557}.

\begin{figure}[t]
    \centering
    \includegraphics[width=0.85\linewidth]{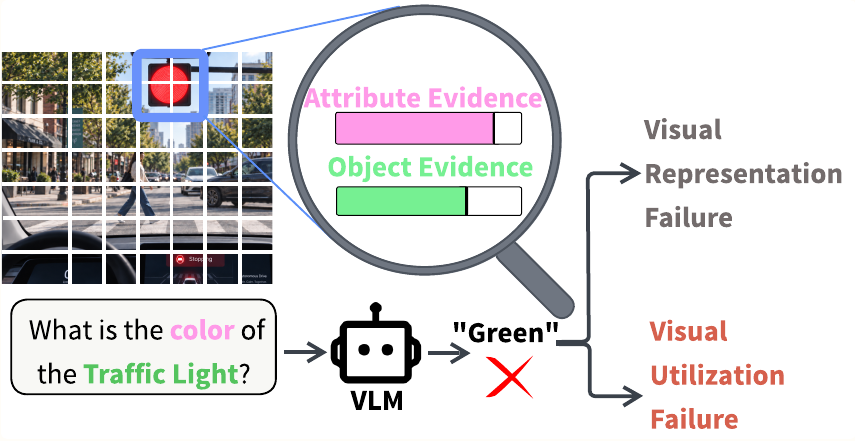}
    \caption{\textbf{Diagnosing the source of errors.}}
    \label{fig:TheGoal}
    \vspace{-0.8em}
\end{figure}

We therefore shift focus from whether a VLM answers correctly to whether it possesses the necessary visual evidence to do so. We quantify the \emph{visual-utilization failure rate}: 
the fraction of errors for which sufficient evidence is encoded but unsuccessfully used. The rest are visual-representation failures, where the evidence was never encoded. 
We estimate this unobservable rate with a \emph{probe} trained to extract a target signal from representations. 
However, it might guess correctly from spurious cues or fail to extract present evidence. 
We account for this uncertainty by developing theoretical bounds on the true rate using answer-preserving transformations that vary evidence salience. 

We introduce \method~(\fullname), a lightweight rank-based visual-evidence probe. We evaluate it on target-object attribute questions (``\emph{What is the color of the traffic light?}''), which require object and attribute evidence. Because visual-token representations are not directly interpretable, 
\method projects each language-model layer's representations into vocabulary space, providing a window into the VLM's encoded visual information. It measures whether tokens representing the target concept (the object or correct attribute value) rank above competing tokens and combines these scores across layers to determine whether object and attribute evidence are present.

To evaluate \method, we use six thousand object-rich images from the Visual Attributes in the Wild (VAW) dataset \citep{DBLP:conf/cvpr/PhamK0DCTS21}, focusing on color, material, and shape as a first step toward broader visual evidence, and two VLMs sharing the causal decoder-based design of current models, LLaVA-1.6-7B \citep{DBLP:conf/cvpr/LiuLLL24} and Qwen2.5-VL-7B \citep{DBLP:journals/corr/abs-2502-13923}.
Across all attribute categories and both VLMs, \method outperforms every baseline, and combined with the strongest object-evidence probe it reaches joint object-and-attribute ROC-AUCs of 0.95 on LLaVA and 0.91 on Qwen. We additionally test its generalization in low-resource and unseen-value settings, where it also outperforms the corresponding baselines, supporting its use for diagnosing error sources.

We then use \method to answer the paper's central question. On VAW, we estimate that 84.4\% of LLaVA errors and 70.0\% of Qwen errors are visual-utilization failures. More importantly, our conservative lower bound shows these are at least 77.5\% (LLaVA) and 61.0\% (Qwen) of errors overall. Thus, the main bottleneck is not representing the visual information, but utilizing it when producing the answer. We find no universal layer at which visual evidence becomes accessible. It depends heavily on the model and attribute. To test whether encoded evidence can be used more successfully, we increase attention toward the target-object region. As predicted, this achieves twice the correction rate on evidence-positive errors as on errors where \method detects no evidence, as proof of concept that, when visual evidence is encoded, helping the model use it can resolve failures.

In summary, our contributions are: (1) we formalize the distinction between visual-representation and visual-utilization failures, introduce an estimand for their relative prevalence, and derive bounds that account for an imperfect probe; (2) we introduce \method, a rank-based probe that traces object and attribute evidence across language-model layers, and characterize its scope and limitations (\S\ref{app:extensions}); (3) across two VLMs and three attributes, we show that visual utilization is the main bottleneck; and (4) we show that identifying whether the required visual evidence is encoded can guide error mitigation. Together, our study suggests that improving VLM reliability requires not only better encoding, but also better mechanisms to locate, route, and use what VLMs already encode.

\section{Related Work}
\label{sec:related}

\paragraph{Detection and Mitigation of Errors in VLMs.}
VLM responses often describe objects, attributes, or relations inconsistent with the image \citep{DBLP:journals/corr/abs-2404-18930, DBLP:journals/ijcv/ChenMZYWWS26, DBLP:journals/corr/abs-2402-00253}. Prior work has mainly focused on detecting such conflicts \citep{DBLP:conf/emnlp/LiDZWZW23, DBLP:conf/aaai/GunjalYB24, DBLP:journals/corr/abs-2402-15300, DBLP:journals/chinaf/YinFZXWSSLSC24}, 
treating VLMs as black boxes \citep{DBLP:conf/iclr/JiangKPG25}. Mitigation methods \citep{DBLP:journals/corr/abs-2401-01313} reduce errors through attention-guided grounding or constrained decoding at inference \citep{DBLP:conf/cvpr/HuangDZ0H0L0Y24, DBLP:conf/cvpr/LengZCLLMB24, DBLP:conf/icml/ZhaoDZG25}, or visual-text alignment objectives during training \citep{DBLP:conf/iclr/LiuLLWYW24, DBLP:conf/cvpr/YuYZHHCHL0024}. 
However, they address the output without identifying its underlying cause. While \citet{DBLP:journals/corr/abs-2603-15557} frame VLM errors as failures of internal processing, they analyze utilization without establishing whether the required visual evidence was encoded. We instead distinguish visual-representation failures, where evidence is not encoded, from visual-utilization failures, where encoded evidence is unsuccessfully used \citep{calderon26a}.

\paragraph{Interpretability of visual information in VLM.}
Motivated by evidence that LLM internal representations reveal information underlying model behavior \citep{DBLP:conf/iclr/OrgadTGRSKB25, DBLP:conf/emnlp/AzariaM23}, recent work studies how visual information is encoded and utilized within VLMs. 
First, \textit{classifier-based} methods, including linear probing \citep{DBLP:conf/iclr/AlainB17, DBLP:journals/corr/abs-2506-08008, DBLP:conf/eacl/KogilathotaGSZ26}, test whether visual concepts can be recovered from hidden states. Notably, \citet{DBLP:journals/corr/abs-2506-08008} shows that information easily recoverable from the intermediate vision encoder layers may fail to influence predictions. 
Second, \textit{confidence-based} methods project hidden states through the LM output head, following the logit lens \citep{nostalgebraist2020logitlens}, to identify vocabulary concepts supported by intermediate visual representations \citep{DBLP:conf/cvpr/JiangCZLSY25, DBLP:conf/iclr/NeoO0G0B25, DBLP:conf/iclr/JiangKPG25}. These studies show that object semantics are encoded within visual tokens and localized to corresponding image regions, become increasingly interpretable across layers, and provide early indications of visually unsupported objects in textual responses. 
Finally, \textit{cross-modality attention analysis} \citep{DBLP:conf/cvpr/KaduriBD25, DBLP:conf/cvpr/JiangCZLSY25, DBLP:conf/eccv/LiuZC24} tracks how visual features contribute to final outputs and how attention disruptions correlate with errors. Unlike per-layer or per-concept analyses, we analyze objects and attributes jointly across layers.

\section{Quantifying the Rate of Visual-Utilization Failures}

Understanding the source of VLM errors requires distinguishing two failure modes: \textbf{visual-representation failures}, in which the required visual evidence is not encoded, and \textbf{visual-utilization failures}, in which it is encoded but unsuccessfully used. Although the final answer alone cannot distinguish them, they reflect different limitations and call for different interventions. We therefore formalize this distinction and propose an estimand, estimator, and bounds for the proportion of errors from each source.

\subsection{Problem Formulation}
\label{subsec:problem_formulation}
For an input $X=(I,Q)$ consisting of an image and a query, 
let $Y_X$ denote its ground-truth answer and $f(X)$ the VLM prediction. 
An error (sometimes referred to as hallucination, and hence the notation) occurs when these answers differ, i.e., $H(X)=\mathbb{I}({f(X)\neq Y_X})$, and we denote the error event by $\mathcal{H}:={H(X)=1}$. Observing an error does not reveal which failure mode caused it: even with access to the raw VLM representations, whether they contain sufficient visual information to support the correct answer is not directly observable. Let $\phi_X\in{0,1}$ represent this latent state, where $\phi_X=1$ indicates that the representations contain sufficient signal to support $Y_X$. Conditional on an error, the presence of such a signal indicates that the evidence was encoded but not used successfully, and hence that the error is a visual-utilization failure. 
We therefore seek the proportion of errors attributable to visual-utilization failures:
\begin{equation*}
\mu_\phi := \mathbb{P}(\phi_X=1 \mid \mathcal{H}).
\end{equation*}
The remaining fraction, $1-\mu_\phi$, corresponds to visual-representation failures. Since $\phi_X$ is latent, we approximate it using a \emph{probe}: a lightweight secondary model trained on the VLM's internal representations to determine whether they contain evidence supporting an answer $Y$. Let $M(X,Y)\in\{0,1\}$ denote the probe output for an input $X$ and a candidate answer $Y$. Thus, $M(X,Y_X)=1$ indicates that the probe can extract evidence supporting the ground-truth answer $Y_X$. This suggests that the VLM's visual representations had the potential to support $Y_X$, but that this evidence was not used successfully. Accordingly, the probe-based estimand and its empirical plug-in estimator are:
\begin{equation*}
\mu_\phi^{\mathrm{probe}}
:= \mathbb{P}\bigl(M(X, Y_X)=1\mid\mathcal{H}\bigr),
\qquad
\widehat{\mu}_\phi
:= \frac{\sum_{i=1}^{n}{\mathbb{I}\{H(x_i)=1,M(x_i, y_i)=1\}}}
{\sum_{i=1}^{n}{\mathbb{I}\{H(x_i)=1\}}}.
\end{equation*}
A perfect probe satisfying $M(X,Y_X)=\phi_X$ for every $X$ makes the probe-based estimand equal $\mu_\phi$. In practice, however, $M$ is not trained on $\phi_X$, which is unobserved, but on labels indicating whether $Y$ is correct for $X$, so $M(X,Y_X)$ is only a proxy: it supports but does not prove $\phi_X$ in either direction, as a positive prediction may rely on imperfect or spurious cues, and a negative one may fail to extract present signal. To account for this uncertainty, we next derive bounds on $\mu_\phi$.

\subsection{Bounds with an Imperfect Probe}
\label{subsec:bounds_imperfect_probe}

Let $X_+=\psi_+(X)$ and $X_-=\psi_-(X)$ denote answer-preserving transformations (i.e., $Y_{X_+}=Y_X=Y_{X_-}$) designed to make the visual evidence relevant to the query more or less salient, respectively. For example, $\psi_+$ may highlight the relevant object with a bounding box, whereas $\psi_-$ may blur or partially occlude it. 
Consider the probe-dominance assumption (Assumption~\ref{asm:probe_dominance}), under which the probe's detection rate after the positive transformation $\psi_+$ and its non-detection rate after the negative transformation $\psi_-$ serve as conservative proxies for its behavior when the original representations genuinely contain or lack the required signal, respectively. Provided that $\mathbb{P}(M(X_+,Y_X)=1\mid\mathcal{H})>0$ and $\mathbb{P}(M(X_-,Y_X)=0\mid\mathcal{H})>0$, the following bounds hold:
\begin{theorem}[Visual-Utilization Failure Bounds. Proof in \S\ref{app:proof}]
\label{thm:vh_bounds}
\begin{equation*}
1-
\frac{\mathbb{P}(M(X,Y_X)=0\mid\mathcal{H})}
{\mathbb{P}(M(X_-,Y_X)=0\mid\mathcal{H})}
\le
\mu_\phi
\le
\frac{\mathbb{P}(M(X,Y_X)=1\mid\mathcal{H})}
{\mathbb{P}(M(X_+,Y_X)=1\mid\mathcal{H})}.
\end{equation*}
\end{theorem}

\subsection{Visual Evidence in Object-Attribute Queries}
\label{subsec:visual_evidence_object_attribute}
We now specialize the probe to queries $Q$ about an attribute of a target object, such as `\emph{What is the \underline{color} of the \underline{car}?}'', where the target object is \emph{``car''} and the attribute category is \emph{``color''}. Correct answering requires two components of visual evidence: \emph{object visual evidence}, indicating that the representation localizes the target object, and \emph{attribute visual evidence}, indicating that the object-specific representation supports the correct attribute value more strongly than alternative values from the same category. We define $M$ as a composite probe:
\begin{equation*}
M(X,Y)=1
\quad\Longleftrightarrow\quad
M_{\mathrm{obj}}(X)=1
\ \land
M_{\mathrm{attr}}(X,Y)=1.
\end{equation*}
Under this operational definition, for the ground-truth answer $Y_X$, at least one required evidence component is not detected when $M(X,Y_X)=0$, whereas both are detected when $M(X,Y_X)=1$. In the next section, we present our implementation of both probes.

\begin{figure}[t]
    \centering
    \includegraphics[width=\linewidth]{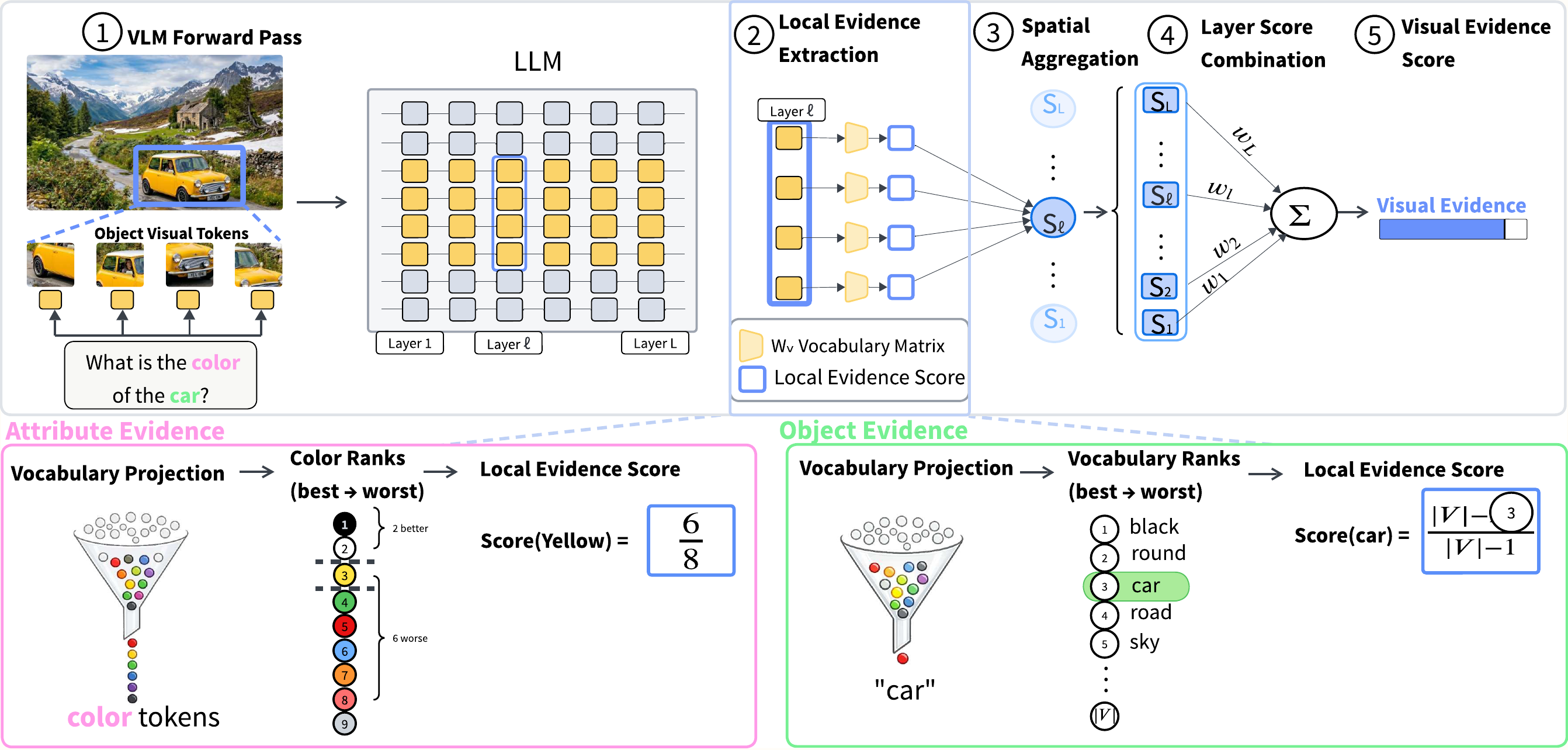}
    \vspace{-0.8em}
    \caption{\textbf{Overview of \method.} \emph{Top:} Visual-token representations from each LLM layer are projected through the output head $W_V$ to obtain vocabulary rankings. These rankings are converted into local evidence scores and aggregated into one score per layer. The layer-wise scores serve as input to a logistic-regression classifier. \emph{Bottom:} For each visual token, the vocabulary ranks of target tokens ($\mathcal{T}^+$) and competing tokens ($\mathcal{T}^-$) are compared to produce a local evidence score.}
    \label{fig:method_overview}
    \vspace{-1.0em}
\end{figure}

\section{The Visual Evidence Lens Method}
\label{sec:method}

We now introduce \method, a \textit{rank-based} implementation of the composite probe $M=M_{\mathrm{obj}}\land M_{\mathrm{attr}}$ (\S\ref{subsec:visual_evidence_object_attribute}) that projects intermediate visual-token representations into the LM's vocabulary space and traces the resulting evidence across layers (Figure~\ref{fig:method_overview}). As we show experimentally, \method yields the most reliable probe and hence our most reliable estimate of $\mu_\phi$.

\subsection{Method Overview and Notation}
\label{subsec:method_overview}
For an input $X=(I,Q)$, the vision encoder transforms the image $I$ into patch-level representations. The vision-to-language projection then maps these representations into the language model's embedding space, producing visual tokens $k_1,\ldots,k_n$. In the autoregressive VLMs considered here, as in most current VLMs, visual tokens precede the question tokens; causal masking therefore prevents their layer-wise representations from attending to $Q$, so they encode only information available from the image. Let $h_l(k_i)\in\mathbb R^d$ denote the representation of visual token $k_i$ at language-model layer $l\in\{1,\ldots,L\}$.

To examine the visual evidence encoded in $h_l(k_i)$, we apply the LM's output head $W_V\in\mathbb R^{|V|\times d}$, obtaining logits over the vocabulary $V$ and hence a ranking of its tokens, following the logit-lens approach \citep{nostalgebraist2020logitlens} (\S\ref{app_sub:logit_lens}). For each $v\in V$, let $\operatorname{rank}_l(k_i,v)$ denote its projected rank, from one (highest scoring) to $|V|$ (lowest scoring). These rankings reveal which vocabulary concepts are most strongly supported by each visual-token representation across LM layers. We use this vocabulary-based lens to quantify visual object and attribute evidence.

Following the example from Section~\ref{subsec:visual_evidence_object_attribute} for the question $Q=$ ``\emph{What is the \underline{color} of the \underline{car}?}'', let $o$ denote the target object (\emph{``car''}), $b$ its annotated bounding box, and $Y$ the target attribute value (\emph{``yellow''}). To evaluate attributes, we use the full set of values in the queried category as the closed candidate set $\mathcal A$ (e.g., $\mathcal A=\{\textrm{yellow},\textrm{red},\textrm{blue}\}$ for \emph{color}); all values other than $Y$ serve as alternatives. Let $\mathcal K_{\mathrm{in}}$ and $\mathcal K_{\mathrm{out}}$ the sets of visual tokens whose corresponding image patches lie inside and outside $b$, respectively.

\subsection{Visual-Evidence Extraction}
\label{subsec:visual_evidence_extraction}

To measure visual evidence for a specific concept, we define target vocabulary tokens $\mathcal T^+$ that represent it and competing tokens $\mathcal T^-$.
For instance, when probing for the target attribute $Y=$ \emph{``yellow''}, $\mathcal T^+$ contains vocabulary tokens for ``yellow'', while $\mathcal T^-$ contains tokens for the alternative candidates in $\mathcal A$ (e.g., ``red'', ``blue''). The exact construction of $\mathcal T^+$ and $\mathcal T^-$ differs for objects and attributes; detailed definitions for each are provided in \S\ref{app_sub:tokens_constraction}.
For visual token $k_i$ and layer $l$, we first take the best rank across the variants in the target token set, $r_l(k_i,\mathcal T^+):=\min_{v\in\mathcal T^+}\operatorname{rank}_l(k_i,v)$, and then compare it with the ranks of the competing tokens:
\begin{equation}
    s_l(k_i)
    :=\frac{1}{|\mathcal T^-|}
      \sum_{u\in\mathcal T^-}
      \mathbb I\!\left\{
          \operatorname{rank}_l(k_i,u)>r_l(k_i, \mathcal T^+)
      \right\}.
    \label{eq:patch_evidence_score}
\end{equation}
With ranks ordered from best to worst starting at one, the resulting local evidence score $s_l(k_i)\in[0,1]$ is the fraction of competing tokens that appear later in this ordering than the best-ranked target token; larger values therefore indicate stronger relative visual evidence at $k_i$.
To obtain a layer-level score, we summarize the local scores $\{s_l(k_i):k_i\in\mathcal K\}$ over a specific spatial region $\mathcal K$ (e.g., $\mathcal K_{\mathrm{in}}$ or $\mathcal K_{\mathrm{out}}$) using the aggregation strategies defined in \S\ref{app:aggregation}. This produces a layer-level score $S_l$; stacking the scores from all $L$ layers yields an evidence vector $\mathbf S := [S_1,\ldots,S_L]^\top\in\mathbb R^L$.

\subsection{Probe Supervision}
\label{subsec:probe_supervision}
To determine the final presence of visual evidence, we train lightweight logistic-regression probes on the extracted $L$-dimensional evidence vectors. Both probes are framed as binary classifiers over $\mathbf S \in \mathbb R^L$, differing only in how the positive ($z=1$) and negative ($z=0$) examples are sourced.

\paragraph{Attribute-Evidence Regression.}
For the attribute visual probe, we extract the evidence vector $\mathbf S^{\mathrm{attr}}(Y) \in \mathbb R^L$ over the target object's region $\mathcal K_{\mathrm{in}}$ for a candidate attribute $Y$. Since visual evidence is not directly observable, we supervise the probe using dataset annotations: each evidence vector is paired with $z=1$ when $Y$ is annotated as present and $z=0$ when it is annotated as absent. We train a logistic-regression model $\mathrm{LR}_{\mathrm{attr}}:\mathbb R^L\rightarrow[0,1]$ on these labeled vectors and apply it to the target-attribute vector $\mathbf S^{\mathrm{attr}}(Y)$:
\begin{equation}
    M_{\mathrm{attr}}(X, Y)
    :=\mathbb I\!\left\{
       \mathrm{LR}_{\mathrm{attr}}
       \bigl(\mathbf S^{\mathrm{attr}}(Y)\bigr)
       \geq\tau_{\mathrm{attr}}
    \right\}.
    \label{eq:attribute_probe}
\end{equation}
Thus, $M_{\mathrm{attr}}(X, Y)=1$ when the layer-wise visual evidence for $Y$ is classified as a present attribute.

\paragraph{Object-Evidence Regression.}
For the object visual probe, we instead use spatial regions for supervision. The inside-region vector $\mathbf S^{\mathrm{obj,in}}\in\mathbb R^L$ is a positive example ($z=1$), while the outside-region vector $\mathbf S^{\mathrm{obj,out}}\in\mathbb R^L$, where the object is absent, is negative ($z=0$). We fit $\mathrm{LR}_{\mathrm{obj}}:\mathbb R^L\rightarrow[0,1]$ to distinguish the regions and define the object probe as:
\begin{equation}
    M_{\mathrm{obj}}(X)
    :=\mathbb I\!\left\{
       \mathrm{LR}_{\mathrm{obj}}\bigl(\mathbf S^{\mathrm{obj,in}}\bigr)
       \geq\tau_{\mathrm{obj}}
    \right\}.
    \label{eq:object_probe}
\end{equation}

\section{Experimental Setup}
\label{sec:experimental}
\paragraph{Baseline Methods.}
We compare our \textit{rank-based} method with three baseline categories: \textit{classifier-based}, \textit{confidence-based}, and \textit{attention-based}. All keep the VLM frozen and use the same mapping from annotated object regions to visual tokens.

\emph{Classifier-based.}
A standard linear-probing baseline \citep{DBLP:conf/iclr/AlainB17} trains a separate linear classifier for each attribute value at each LM layer. Target-region visual tokens are spatially aggregated (\S\ref{app:aggregation}) into a full-hidden-dimension vector and used as classifier input. This approach requires knowing the attribute or object space in advance, as each value has its own trained probing classifier. Consequently, it cannot be applied to unseen attributes or used as an object probe.

\emph{Confidence-based.}
Internal Confidence \citep{DBLP:conf/iclr/JiangKPG25} measures whether a concept is represented at an intermediate layer by projecting the representation into vocabulary space, and examining the probability assigned to the tokens corresponding to that concept. The method takes the maximum probability across tokens and layers and passes the scalar to a binary classifier. Unlike \method, it does not compare to scores of competing tokens and collapses layers into a single scalar. In our evaluation, we restrict the visual tokens to the annotated target-object region.

\emph{Attention-based.}
High attention between a query token representing a concept and visual tokens corresponding to that concept can indicate stronger visual grounding \citep{DBLP:conf/cvpr/KaduriBD25,DBLP:conf/cvpr/JiangCZLSY25,DBLP:conf/eccv/LiuZC24}. In our implementation, we measure attention from the candidate attribute or target-object query tokens to visual tokens inside and outside the annotated target-object region. At each layer and head, the attention scores are normalized by region size, aggregated over heads within each layer, concatenated across layers, and passed to a binary classifier (\S\ref{app_sub:attention_baseline}).

\paragraph{Pooling Configurations.}
For attribute evidence, \method and the confidence and attention baselines use two configurations: \emph{All}, one probe shared across all values in a category, and \emph{Separate} (Sep), one probe per value. The classifier baseline always trains one probe per value, so it has only Separate. For object evidence, the methods use one estimator shared across object categories. We select the aggregation strategy of scores (mean, max, min or median) on a validation subset (\S\ref{app_sub:dev_set_results}), retrain on the full training set, and evaluate once on held-out test set.

\paragraph{Base Models.}
We evaluate two VLMs: LLaVA-1.6-7B \cite{DBLP:conf/cvpr/LiuLLL24}, which connects a pretrained CLIP vision encoder to a Mistral-7B LM, and Qwen2.5-VL-7B \cite{DBLP:journals/corr/abs-2502-13923}, which pairs the Qwen2.5 LM with a natively trained dynamic-resolution vision encoder.

\paragraph{Dataset.}
We use the Visual Attributes in the Wild (VAW) dataset \citep{DBLP:conf/cvpr/PhamK0DCTS21}, which provides each object with a polygon or bounding box and associated attributes. This allows us to identify the target object's visual region and construct object-attribute questions (\S\ref{app_subsec:data preprocessing}).
VAW contains natural, object-rich scenes with over 17 objects per image on average.
We normalize attribute values to canonical forms (\S\ref{app_subsec:data preprocessing}) and sample 6,074 images.
We focus on color, material, and shape because they can be determined from the target object itself rather than its surrounding context, making them well suited for identifying visual-representation failures. \method itself is not tied to this setting: \S\ref{app:extensions} discusses its scope and limitations.

\section{Results}

\subsection{Evaluating the Visual-Evidence Probe}
\label{subsec:evaluating_visual_evidence_probe}

We estimate the visual-utilization failure rate as the fraction of errors, $H(X)=1$, for which $M(X,Y_X)=1$: the VLM answers incorrectly although both required evidence components can be extracted from its representations. Because this estimate's reliability depends on how well $M$ detects supporting visual evidence, we first compare our rank-based probe with the baselines. We report ROC-AUC here and additional metrics and aggregation strategies in \S\ref{app_sub:dev_set_results}.

\begin{table}[t]
  \centering
  \caption{\textbf{Visual-evidence detection performance.} ROC-AUC for attribute, object, and combined probes across methods and VLMs.}
  \label{tab:evidence_detection_auc}
  \begin{adjustbox}{width=0.98\textwidth}
  \begin{tabular}{ll|ccccc|ccccc}
  \toprule
  \cellcolor{Gray!15}
  & \cellcolor{Gray!15}
  & \multicolumn{5}{c|}{\cellcolor{Gray!40}\textbf{LLaVA}}
  & \multicolumn{5}{c}{\cellcolor{Gray!40}\textbf{Qwen}} \\
  \cellcolor{Gray!15} $M_{\mathrm{attr}}$
  & \cellcolor{Gray!15} $M_{\mathrm{obj}}$
  & \cellcolor{Gray!15} \shortstack{Color\\{\tiny $N=1{,}280$}}
  & \cellcolor{Gray!15} \shortstack{Mater\\{\tiny $N=1{,}048$}}
  & \cellcolor{Gray!15} \shortstack{Shape\\{\tiny $N=800$}}
  & \cellcolor{Gray!15} \shortstack{Object\\{\tiny $N=2{,}941$}}
  & \cellcolor{Gray!15} \shortstack{Attr$\times$Obj\\{\tiny $N=3{,}128$}}
  & \cellcolor{Gray!15} \shortstack{Color\\{\tiny $N=1{,}280$}}
  & \cellcolor{Gray!15} \shortstack{Mater\\{\tiny $N=1{,}048$}}
  & \cellcolor{Gray!15} \shortstack{Shape\\{\tiny $N=800$}}
  & \cellcolor{Gray!15} \shortstack{Object\\{\tiny $N=2{,}941$}}
  & \cellcolor{Gray!15} \shortstack{Attr$\times$Obj\\{\tiny $N=3{,}128$}} \\
  \midrule
  Rank-All
  & Rank
  & 0.95 & 0.91 & 0.84 & 0.90 & 0.90
  & 0.93 & 0.87 & 0.82 & 0.81 & 0.86 \\
  Rank-Sep
  & Rank
  & \textbf{0.97} & \textbf{0.97} & \textbf{0.91} & 0.90 & 0.94
  & \textbf{0.94} & \textbf{0.94} & \textbf{0.87} & 0.81 & 0.90 \\
  Conf-All
  & Conf
  & 0.79 & 0.83 & 0.64 & 0.53 & 0.78
  & 0.78 & 0.82 & 0.67 & 0.58 & 0.77 \\
  Conf-Sep
  & Conf
  & 0.79 & 0.83 & 0.64 & 0.53 & 0.80
  & 0.78 & 0.82 & 0.67 & 0.58 & 0.78 \\
  Classifier
  & --
  & 0.93 & 0.94 & 0.83 & -- & --
  & 0.92 & \textbf{0.94} & 0.86 & -- & -- \\
  Attention-All
  & Attention
  & 0.81 & 0.75 & 0.67 & \textbf{1.00} & 0.76
  & 0.82 & 0.77 & 0.72 & \textbf{0.96} & 0.78 \\
  Attention-Sep
  & Attention
  & 0.83 & 0.78 & 0.69 & \textbf{1.00} & 0.78
  & 0.82 & 0.77 & 0.73 & \textbf{0.96} & 0.78 \\
  \midrule
  \rowcolor{Gray!15}
  \textbf{Rank-Sep}
  & \textbf{Attention}
  & \textbf{0.97} & \textbf{0.97} & \textbf{0.91} & \textbf{1.00} & \textbf{0.95}
  & \textbf{0.94} & \textbf{0.94} & \textbf{0.87} & \textbf{0.96} & \textbf{0.91} \\
  \bottomrule
  \end{tabular}
  \end{adjustbox}
  \end{table}

\paragraph{Our Method Provides the Strongest Attribute Evidence Detection.} 
Table~\ref{tab:evidence_detection_auc} compares the performance of probes across attribute, object, and combined visual evidence detection. For attribute evidence detection, Rank-Separate performs best in all three categories on both LLaVA and Qwen. The Attention baseline is strongest for object evidence, with AUCs of 1.00 on LLaVA and 0.96 on Qwen; our Rank-based is second at 0.90 and 0.81, respectively. Combining Rank-Separate as $M_{\mathrm{attr}}$ with Attention as $M_{\mathrm{obj}}$ gives the best combined AUC: \textbf{0.95} on LLaVA and \textbf{0.91} on Qwen.

\paragraph{Our Method Generalizes with Limited Supervision and to Unseen Attribute Values.} 
For $M_{\mathrm{attr}}$, \method uses two pooling configurations: \emph{Rank-Separate} (Sep) trains one probe per value, while \emph{Rank-All} shares one across values. Our main results evaluate attribute values seen during training, where Rank-Sep performs best among methods. We next ask how the probe can be extended to a new value. For example, after training on seven colors, how can it detect evidence for an eighth? We consider two options: train a value-specific probe from a few labeled examples (low-resource), or reuse probes trained on known values for out-of-distribution generalization.

In the low-resource setting, we train a value-specific Separate probe with varying numbers of labeled examples. Figure~\ref{fig:low_resource_unseen_values} (left) reports ROC-AUC averaged over five independently sampled subsets for each training size in each of the three attribute categories. \method outperforms the baselines at every size and remains strong with only a few dozen examples per class. For the unseen-value setting, we hold out one value, train on the remaining $n-1$, and compute its evidence vector. 
Separate applies the $n-1$ value-specific probes to this vector and averages their predictions, while All applies one probe trained jointly on the $n-1$ values. The middle and right panels show that Rank-All generalizes better to unseen values, while both configurations outperform their baselines. Thus, specialization is preferable in-domain, whereas sharing across values improves transfer: a new value can be handled with a few labeled examples using Separate, or with All when no labels are available.


\begin{figure}[t]
    \centering
    \includegraphics[width=\linewidth]{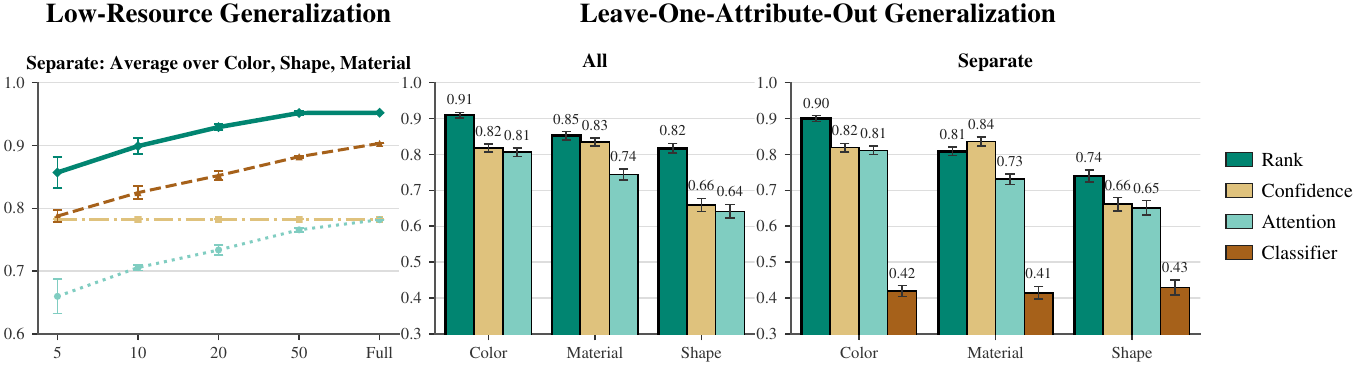}
    \vspace{-1.5em}
    \caption{\textbf{Low-resource and unseen-value generalization.}
    \emph{Left:} ROC-AUC as the number of labeled examples per class varies, using Separate and averaging across color, material, and shape, with five independently sampled subsets for each size and category.
    \emph{Middle and right:} Leave-one-value-out ROC-AUC, averaged over held-out values, using both pooling configurations. For Separate, individual probes are trained on each seen value, and their predictions are averaged for the held-out value. For All, one probe is trained jointly on all seen values and applied to the held-out value. Error bars are computed via 100 bootstraps.}
    \label{fig:low_resource_unseen_values}
\end{figure}

\begingroup
\raggedbottom

\begin{figure*}[t]
\centering
\setlength{\abovecaptionskip}{3pt}

\newcommand{\failureestimaterow}[6]{%
  \pgfmathsetmacro{\representationestimate}{100-#3}%
  \pgfmathsetmacro{\redcenter}{#3/2}%
  \pgfmathsetmacro{\graycenter}{(#3+100)/2}%
  \pgfmathsetmacro{\boundleft}{#4}%
  \pgfmathsetmacro{\boundright}{#5}%

  \node[
    anchor=east,
    font=\small
  ] at (-3,#2) {#1};

  \fill[utilizationfailurered]
    (0,{#2-0.30}) rectangle
    (#3,{#2+0.30});

  \fill[visualfailuregray]
    (#3,{#2-0.30}) rectangle
    (100,{#2+0.30});

  \draw[
    black!70,
    line width=#6
  ]
    (0,{#2-0.30}) rectangle
    (100,{#2+0.30});

  \draw[
    black!80,
    line width=0.95pt
  ]
    (#3,{#2-0.30}) --
    (#3,{#2+0.30});

  \draw[
    failureboundgray,
    line width=0.65pt
  ]
    (\boundleft,{#2-0.48}) --
    (\boundright,{#2-0.48});
  \draw[failureboundgray,line width=0.65pt]
    (\boundleft,{#2-0.57}) -- (\boundleft,{#2-0.39})
    (\boundright,{#2-0.57}) -- (\boundright,{#2-0.39});
  \fill[failureboundgray]
    (#3,{#2-0.48}) circle[radius=0.75pt];

  \node[
    font=\fontsize{7.2}{7.6}\selectfont
  ] at (\redcenter,#2)
    {\textbf{#3\%}};

  \node[
    font=\fontsize{6.2}{6.6}\selectfont
  ] at (\graycenter,#2)
    {%
      \textbf{%
        \pgfmathprintnumber[
          fixed,
          fixed zerofill,
          precision=1
        ]{\representationestimate}\%
      }
    };

  \node[
    font=\fontsize{5.8}{6.2}\selectfont,
    anchor=north,
    inner sep=0pt,
    text=black!80
  ] at (\redcenter,{#2-0.40})
    {[#4--#5]\%};
}

\newcommand{\failurecolumnheaders}{%
  \node[
    anchor=south west,
    font=\fontsize{7.5}{8}\selectfont,
    align=left,
    inner sep=0pt,
    text=utilizationfailurered
  ] at (0,4.95)
    {Utilization\\failures};

  \node[
    anchor=south east,
    font=\fontsize{7.5}{8}\selectfont,
    align=right,
    inner sep=0pt,
    text=black!75
  ] at (100,4.95)
    {Representation\\failures};
}

\begin{minipage}[t]{\textwidth}
  \vspace{0pt}
  \centering

  \textbf{(a) Controlled transformations}\par
  \vspace{-0.2em}

  \includegraphics[width=\linewidth,trim=0 84 0 0,clip]{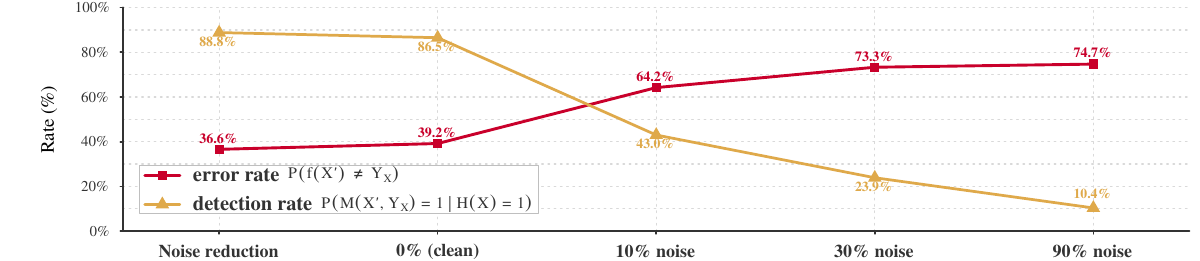}\\[0pt]
  \includegraphics[width=\linewidth,trim=0 12 0 132,clip]{figures/noise_failure_evidence_with_images.pdf}
\end{minipage}


\begin{minipage}[t]{\textwidth}
  \vspace{0pt}
  \centering

  \textbf{(b) Estimated error sources}\par
  \vspace{0.1em}

  \resizebox{\linewidth}{!}{%
    \begin{tikzpicture}[
      x=0.55mm,
      y=5.5mm,
      font=\sffamily
    ]

      \begin{scope}[xshift=1.1cm]

        \node[
          font=\normalsize\bfseries
        ] at (50,5.55)
          {LLaVA};

        \failurecolumnheaders

        \failureestimaterow
          {\textbf{Overall}}
          {4.45}
          {84.4}
          {77.5}
          {95.4}
          {1.05pt}

        \failureestimaterow
          {Color}
          {2.95}
          {86.5}
          {82.2}
          {97.3}
          {0.35pt}

        \failureestimaterow
          {Material}
          {1.45}
          {84.1}
          {74.7}
          {93.6}
          {0.35pt}

        \failureestimaterow
          {Shape}
          {-0.05}
          {81.3}
          {71.3}
          {94.0}
          {0.35pt}

      \end{scope}

      \begin{scope}[xshift=8.3cm]

        \node[
          font=\normalsize\bfseries
        ] at (50,5.55)
          {Qwen};

        \failurecolumnheaders

        \failureestimaterow
          {\textbf{Overall}}
          {4.45}
          {70.0}
          {61.0}
          {96.6}
          {1.05pt}

        \failureestimaterow
          {Color}
          {2.95}
          {67.9}
          {59.3}
          {94.6}
          {0.35pt}

        \failureestimaterow
          {Material}
          {1.45}
          {73.2}
          {60.9}
          {97.9}
          {0.35pt}

        \failureestimaterow
          {Shape}
          {-0.05}
          {69.5}
          {62.5}
          {97.5}
          {0.35pt}

      \end{scope}

    \end{tikzpicture}%
  }
\end{minipage}

\caption{\emph{Top:} The object region is emphasized by blurring its surroundings (leftmost) or degraded by Gaussian noise, yielding the transformed image $X'$ (\S\ref{app:controlled_transformations}). Lines report the LLaVA error rate (red) and joint visual-evidence detection among the original errors $\mathbb{P}(M(X',Y_X)\!=\!1\!\mid\! H(X)\!=\!1)$ (yellow). \emph{Bottom:} \method estimates of visual-utilization (red) and visual-representation (gray) failures. The lines and bracketed values are the utilization-failure rate bounds.}
  \label{fig:failure_analysis}
\vspace{-0.8em}
\end{figure*}
\endgroup

\subsection{Utilization Failures Are the Main Bottleneck}
\label{subsec:utilization_bottleneck}
We now ask whether VLM errors arise because the required evidence is missing (\emph{visual-representation failure}) or encoded but unused (\emph{visual-utilization failure})? Answering this question reveals where the current bottleneck lies and where future interventions should focus. For errors $H(X)=1$, the fraction with $M(X,Y_X)=1$ (where both attribute and object visual evidence are detected) estimates the utilization-failure rate, $\widehat{\mu}_\phi=\mathbb{P}(M(X,Y_X)=1\mid H(X)=1)$. Because $M$ only approximates latent state $\phi_X$, we report the bounds (Section~\ref{subsec:bounds_imperfect_probe}) computed with answer-preserving transformations. We first verify that these transformations affect the probe as expected.

\paragraph{\method Responds Predictably to Changes in Visual Evidence.}
Figure~\ref{fig:failure_analysis} (top) illustrates the transformations used to construct the bounds. The negative transformation $\psi_-(X)$ adds noise to the target object, while the positive transformation $\psi_+(X)$ emphasizes it by blurring its surroundings. For LLaVA, increasing noise from 0\% to 90\% progressively decreases evidence detection among the original errors from \textbf{86.5\%} to \textbf{10.4\%}, while the overall error rate increases from \textbf{39.2\%} to \textbf{74.7\%}. Emphasizing the object produces the reverse pattern: evidence detection increases to \textbf{88.8\%}, while the error rate decreases to \textbf{36.6\%}. Thus, \method responds predictably to changes in the relevant visual signal. Using the validation set, we select \textbf{30\% noise} as $\psi_-(X)$, as it weakens the signal while preserving the ground-truth answer. The positive transformation yields only a modest improvement, making the upper bound less informative. Future work could explore stronger answer-preserving positive transformations to obtain tighter upper bounds.

\paragraph{Visual-Utilization Failures Dominate VLM Errors.}
Figure~\ref{fig:failure_analysis} (bottom) reports $\widehat{\mu}_\phi$ overall and by attribute for both models. The estimated utilization-failure rate is 84.4\% for LLaVA and 70.0\% for Qwen, with lower bounds of 77.5\% and 61.0\%, respectively, and the pattern is consistent across attributes. 
Together, these findings provide a clear, consistent picture: in the evaluated setting, the main bottleneck is properly using visual information during generation, not encoding it.

\subsection{Tracing Visual Evidence and Improving Its Utilization}
\label{subsec:tracing_and_improving_evidence}

Having established that most errors occur despite successfully encoded visual evidence, we next ask two questions: where does this evidence become accessible within the model, and can we improve its utilization once it is encoded? \method naturally supports both analyses: applying it to layers traces how evidence develops across the model, while its final prediction identifies errors in which the required evidence is encoded.

\paragraph{Visual Evidence Emerges Differently Across Attributes and Models.} 
Although \method combines scores across all layers, we train per-layer classifiers using each layer's evidence score separately to trace evidence accessibility with depth. Figure~\ref{fig:layerwise_visual_evidence} shows that evidence generally strengthens toward upper layers, consistent with prior work finding semantic information increasingly accessible in middle-to-late layers \citep{DBLP:conf/iclr/NeoO0G0B25,DBLP:conf/iclr/JiangKPG25,DBLP:conf/cvpr/JiangCZLSY25}. However, trajectories vary across attributes and models. In Qwen, material evidence is already strong in early layers, whereas color and object evidence rise sharply near the end; in LLaVA, object evidence increases more gradually. These results indicate that visual evidence does not become accessible at a universal layer. Interventions aimed at improving visual utilization may therefore need to account for these layer-wise dynamics rather than targeting a fixed layer across models and concepts.

\paragraph{Evidence-Positive Errors Are More Easily Corrected.}
Distinguishing visual-representation from visual-utilization failures is important, as they call for different interventions. Since our results identify visual utilization as the primary bottleneck, we test whether strengthening access to encoded visual evidence can recover incorrect answers. Guided by the absence of a universal evidence layer, we apply an attention intervention at every layer that increases attention from each text token to visual tokens inside the annotated target-object region (\S\ref{app_sub:attention-boost}). Because this requires a known object region, it serves as a proof of concept rather than a standalone correction method. Figure~\ref{fig:attention_redirection_results} shows an example where \method detects the evidence and the intervention corrects the answer.

\begin{wrapfigure}{r}{0.5\textwidth}
    \centering
    \vspace{-2.5em}
    \includegraphics[width=0.98\linewidth,trim=0 0 0 11,clip]{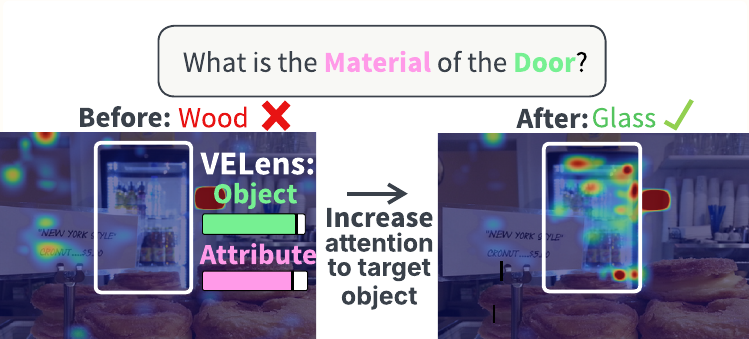}
    \setlength{\abovecaptionskip}{2pt}
    \vspace{-1.0em}
    \caption{An example of evidence-positive error, where \method detects both object and attribute evidence. Increasing attention to the target-object region changes the answer from \emph{wood} to \emph{glass}. More examples in Figure~\ref{fig:app_attention_redirection}.}
    \label{fig:attention_redirection_results}
\end{wrapfigure}
We apply the same intervention to evidence-positive errors, $M(X,Y_X)=1$, and evidence-negative errors, $M(X,Y_X)=0$, obtaining correction rates of 25.2\% and 12\%, respectively. The higher correction rate supports our diagnosis: the intervention is substantially more effective when the required evidence is encoded. This is encouraging because most errors are evidence-positive. It suggests that a large fraction of current errors may be addressed not by replacing the visual encoder, but by developing better mechanisms to locate, route,
and use evidence the model already represents.

\section{Discussion}
\label{sec:discussion}


In this work, we proposed \method, a rank-based probe that traces object and attribute visual evidence across the layers of a VLM. Using it to separate visual-representation from visual-utilization failures, we found that utilization failures account for the majority of errors in both models we studied, with conservative lower bounds of 77.5\% for LLaVA and 61.0\% for Qwen. Our analysis covers attribute questions over color, material, and shape. \method itself, however, operates on vocabulary-space rankings and needs only a target concept and competing alternatives, so it naturally extends to other question types: attribute verification questions, for instance, are answered from evidence for the object and the queried value, both of which \method already provides. More broadly, we see this work as a step toward interpretability of the concepts encoded in visual tokens, and we hope it can extend beyond the question types we study (\S\ref{app:extensions}) and guide model development rather than only describe model behavior: knowing whether the required evidence is encoded determines whether progress should come from a better visual encoder or from better mechanisms to locate, route, and use what the model already represents. 

\subsection*{AI Use Statement}

In this work, we used generative AI tools to implement methods and to reformat a dataset: Claude Opus 5 assisted in writing our code, and GPT-5.6 proposed the canonical attribute values to which we map VAW's fine-grained annotations in each category (Table~\ref{tab:app_canonical_map}). We have not used generative AI tools for generating data, developing theory or proofs, proposing hypotheses, designing or critiquing our methodology and experiments, or interpreting results. Translation and qualitative data analysis are not applicable to this work. Additionally, we used Claude Opus 5 and GPT-5.6 to draft and edit text, and Gemini 3.1 to generate the illustrative images in Figures~\ref{fig:TheGoal} and~\ref{fig:method_overview} (all other images are VAW examples). We have reviewed all AI-assisted work. In all instances, every suggestion or line of code generated by a model was carefully reviewed by the authors to ensure it aligned with our original intentions before being accepted. We take responsibility for the final content of this work, including text, claims or artifacts produced with the aid of generative AI.

\subsection*{Ethics Statement}

This is an interpretability study of publicly released, open-weight VLMs on the public VAW dataset, involving no human subjects, no personally identifying data, and no new data release.

\subsection*{Reproducibility Statement}

To ensure reproducibility, all code for \method, the baselines, and the evaluation pipeline will be made publicly available. Our experiments use two open-weight VLMs, LLaVA-1.6-7B and Qwen2.5-VL-7B, applied frozen, and the publicly available VAW dataset. Data filtering, normalization, and question construction are specified in \S\ref{app_subsec:data preprocessing}, with the exact prompts in Figure~\ref{fig:app_prompt}; the baselines and their configurations in \S\ref{sec:experimental}, \S\ref{app_sub:attention_baseline}, and \S\ref{app_sub:logit_lens}; the design choices of \method, including token-set construction, aggregation strategies, and the development-set protocol used to select them before the single held-out test evaluation, in \S\ref{app_sub:tokens_constraction}, \S\ref{app:aggregation}, and \S\ref{app_sub:dev_set_results}; the controlled transformations underlying our bounds in \S\ref{app:controlled_transformations}; and a complete proof of Theorem~\ref{thm:vh_bounds} with its assumptions in \S\ref{app:proof}.

\bibliography{iclr2027_conference}
\bibliographystyle{iclr2027_conference}

\appendix
\section{Proof of the Visual-Utilization Failure Bounds}
\label{app:proof}

Here we formalizes the probe-dominance assumption used in Section~\ref{subsec:bounds_imperfect_probe} and proves Theorem~\ref{thm:vh_bounds}. Recall that $\phi_X=1$ indicates that the representations of $X$ contain sufficient signal to support the ground-truth answer. Conditional on an error, this latent state corresponds to a visual-utilization failure, whereas $\phi_X=0$ corresponds to a visual-representation failure. The fixed binary probe $M$ is the composite probe defined in Section~\ref{subsec:visual_evidence_object_attribute}; we apply the same probe to $X$, $X_+$, and $X_-$ with respect to the ground-truth answer $Y_X$. The positive transformation $X_+$ strengthens the relevant visual signal, while the negative transformation $X_-$ weakens it, without changing the ground-truth answer.

\begin{assumption}[Probe dominance]
\label{asm:probe_dominance}
Conditioned on an error, probe detection of $Y_X$ on $X_+$ is a conservative proxy for probe detection when the original representations contain the required signal. Similarly, probe non-detection of $Y_X$ on $X_-$ is a conservative proxy for probe non-detection when that signal is absent. Formally,
\begin{align*}
\mathbb{P}\bigl(M(X,Y_X)=1\mid\mathcal H,\phi_X=1\bigr)
&\geq \mathbb{P}\bigl(M(X_+,Y_X)=1\mid\mathcal H\bigr),\\
\mathbb{P}\bigl(M(X,Y_X)=0\mid\mathcal H,\phi_X=0\bigr)
&\geq \mathbb{P}\bigl(M(X_-,Y_X)=0\mid\mathcal H\bigr).
\end{align*}
\end{assumption}
The intuition behind this assumption is that the latent state $\phi_X$ directly describes the property that the probe is designed to detect, whereas $X_+$ and $X_-$ modify this property only indirectly and may do so imperfectly. If the original representations truly contain the required evidence ($\phi_X=1$), the probe should be at least as likely to detect it as it is after the positive transformation. Conversely, if the representations truly lack the required evidence ($\phi_X=0$), the probe should be at least as likely to return non-detection as it is after the negative transformation. Thus, the transformed inputs provide conservative proxies for the latent states.

\begin{proof}[Proof of Theorem~\ref{thm:vh_bounds}]
All probabilities below are conditioned on the error event $\mathcal H$. By definition,
\[
\mu_\phi=\mathbb{P}(\phi_X=1\mid\mathcal H),
\qquad
1-\mu_\phi=\mathbb{P}(\phi_X=0\mid\mathcal H).
\]

First consider the probability that the probe does not detect the required evidence. By the law of total probability,
\begin{align*}
\mathbb{P}\bigl(M(X,Y_X)=0\mid\mathcal H\bigr)
={}&
\mathbb{P}\bigl(M(X,Y_X)=0\mid\mathcal H,\phi_X=1\bigr)\mu_\phi\\
&+\mathbb{P}\bigl(M(X,Y_X)=0\mid\mathcal H,\phi_X=0\bigr)(1-\mu_\phi).
\end{align*}
The first term is nonnegative, and Assumption~\ref{asm:probe_dominance} lower-bounds the second conditional non-detection rate by the non-detection rate on $X_-$. Therefore,
\[
\mathbb{P}\bigl(M(X,Y_X)=0\mid\mathcal H\bigr)
\geq
\mathbb{P}\bigl(M(X_-,Y_X)=0\mid\mathcal H\bigr)(1-\mu_\phi).
\]
When $\mathbb{P}(M(X_-,Y_X)=0\mid\mathcal H)>0$, rearranging gives the lower bound
\[
\mu_\phi
\geq
1-
\frac{\mathbb{P}(M(X,Y_X)=0\mid\mathcal H)}
     {\mathbb{P}(M(X_-,Y_X)=0\mid\mathcal H)}.
\]

For the upper bound, apply the same decomposition to probe detection:
\begin{align*}
\mathbb{P}\bigl(M(X,Y_X)=1\mid\mathcal H\bigr)
={}&
\mathbb{P}\bigl(M(X,Y_X)=1\mid\mathcal H,\phi_X=1\bigr)\mu_\phi\\
&+\mathbb{P}\bigl(M(X,Y_X)=1\mid\mathcal H,\phi_X=0\bigr)(1-\mu_\phi).
\end{align*}
The second term is nonnegative, and Assumption~\ref{asm:probe_dominance} lower-bounds the first conditional detection rate by the detection rate on $X_+$. Hence,
\[
\mathbb{P}\bigl(M(X,Y_X)=1\mid\mathcal H\bigr)
\geq
\mathbb{P}\bigl(M(X_+,Y_X)=1\mid\mathcal H\bigr)\mu_\phi.
\]
When $\mathbb{P}(M(X_+,Y_X)=1\mid\mathcal H)>0$, rearranging yields
\[
\mu_\phi
\leq
\frac{\mathbb{P}(M(X,Y_X)=1\mid\mathcal H)}
     {\mathbb{P}(M(X_+,Y_X)=1\mid\mathcal H)}.
\]
Combining the two inequalities proves Theorem~\ref{thm:vh_bounds}.
\end{proof}

\section{\method Method: Complementary Details}

\subsection{Object and Attribute Token Set Construction}
\label{app_sub:tokens_constraction}

In Section~\ref{subsec:visual_evidence_extraction} we defined the local evidence score $s_l(k_i)$
(Eq.~\ref{eq:patch_evidence_score}) in terms of a target token set $\mathcal{T}^+$ and a competing
token set $\mathcal{T}^-$. Here we detail their construction, which differs for attributes and objects.

\paragraph{Attribute Token Sets.}
For the attribute probe, visual evidence is measured relatively within a category. Because language
model tokenizers are sensitive to capitalization and whitespace, we expand each concept into several
variants. For a concept string $w$ (e.g., ``yellow''), the variant set $\mathcal{T}(w)$ collects the
token IDs of its four variants:
\begin{equation}
    \mathcal{T}(\text{yellow}) = \{ \text{yellow}, \text{Yellow}, \text{\textvisiblespace yellow}, \text{\textvisiblespace Yellow} \}.
\end{equation}
Given the target attribute value $Y$ and the candidate set $\mathcal A$ of the queried category, the
variants of $Y$ form the target set and the variants of the alternative values form the competing set:
\begin{equation}
    \mathcal{T}^+_{\mathrm{attr}} := \mathcal{T}(Y), \quad 
    \mathcal{T}^-_{\mathrm{attr}} := \bigcup_{a \in \mathcal{A} \setminus \{Y\}} \mathcal{T}(a).
\end{equation}
With these sets, $s_l(k_i)$ is the fraction of alternative values in the same category ranked strictly
worse than the best-ranked variant of $Y$.

\paragraph{Object Token Sets.}
For the object probe there is no closed set of alternatives, so visual evidence is measured as
absolute vocabulary support. Using the same expansion for the target object $o$ (e.g., ``car''), its
variants form the target set and the entire remainder of the language model's vocabulary $V$ forms
the competing set:
\begin{equation}
    \mathcal{T}^+_{\mathrm{obj}} := \mathcal{T}(o), \quad 
    \mathcal{T}^-_{\mathrm{obj}} := V \setminus \mathcal{T}^+_{\mathrm{obj}}.
\end{equation}
Because $\mathcal{T}^-_{\mathrm{obj}}$ is the rest of the vocabulary, $s_l(k_i)$ becomes the fraction
of all vocabulary tokens that the best-ranked variant of $o$ outranks, quantifying how strongly
visual token $k_i$ supports $o$ compared to all other vocabulary tokens.

\subsection{Aggregation Strategies}
\label{app:aggregation}

In \method, the local evidence scores of Eq.~\ref{eq:patch_evidence_score} are summarized into a single layer-level score. Formally, for a set of visual tokens $\mathcal{K}$ (e.g., $\mathcal{K}_{\mathrm{in}}$ or $\mathcal{K}_{\mathrm{out}}$), we define the layer-level score $S_l$ at language-model layer $l$ as:
\begin{equation}
S_l = \mathrm{Agg}\big(\{s_l(k_i) : k_i \in \mathcal K\}\big),
\end{equation}
where $s_l(k_i)\in[0,1]$ is the local evidence score of visual token $k_i$ at layer $l$.

\paragraph{Aggregation Operators.}
We consider 4 different operators for $\mathrm{Agg}$:
\begin{itemize}
    \item \textbf{Mean:} 
    \[
    \mathrm{Agg} = \frac{1}{|\mathcal{K}|} \sum_{k_i \in \mathcal{K}} s_l(k_i)
    \]
    \item \textbf{Median:} 
    \[
     \mathrm{Agg} =\mathop{\mathrm{median}}_{k_i \in \mathcal{K}} s_l(k_i)
    \]
    \item \textbf{Minimum:} 
    \[
    \mathrm{Agg} = \min_{k_i \in \mathcal{K}} s_l(k_i)
    \]
    \item \textbf{Maximum:} 
    \[
    \mathrm{Agg} = \max_{k_i \in \mathcal{K}} s_l(k_i)
    \]
\end{itemize}

\paragraph{Attention Head Aggregation.}
The attention baseline (\S\ref{app_sub:attention_baseline}) aggregates over attention heads rather than over visual tokens, as its scores are defined per head and per concept token rather than per visual token. Within layer $l$, the concept $c$ being probed, either a candidate attribute or the target object, is represented in the query by the tokens $q_1,\ldots,q_m$, and every head $h\in\{1,\ldots,H\}$ assigns a separate score to each of them. We first average over these $m$ concept tokens, giving one score per head, and then apply $\mathrm{Agg}$ over the $H$ heads to obtain the layer-level score:
\begin{equation}
S_l = \operatorname{Agg}_{h=1}^{H}
\left(\frac{1}{m}\sum_{j=1}^{m} p_{l,h,j}(c)\right),
\label{eq:head_aggregation}
\end{equation}
where $p_{l,h,j}(c)$ denotes the per-head attention score of concept token $q_j$: the inside attention proportion $p_{l,h,j}^{\mathrm{in}}(a)$ for attributes (Eq.~\ref{eq:attribute_attention_aggregation}), and the region densities $d_{l,h,j}^{\mathrm{in}}(o)$ and $d_{l,h,j}^{\mathrm{out}}(o)$ for objects (Eq.~\ref{eq:object_attention_aggregation}). For this baseline we consider the mean, median, and maximum operators.

\paragraph{Discussion.}
These different aggregation strategies capture different aspects of the visual evidence distribution across visual tokens. The mean reflects overall support, the median provides robustness to outliers, and the minimum and maximum emphasize extreme responses.

\section{Experimental Setup: Complementary Details}
\label{app:experimental_setup}

\subsection{Attention Baseline}
\label{app_sub:attention_baseline}

The attention baseline follows the probe construction in
Section~\ref{sec:method}, but replaces vocabulary-rank evidence with
cross-modal attention. For each input $X=(I,Q)$, the query $Q$ names the
target object $o$ and lists the values in $\mathcal A$, the full set of candidate
values for the queried attribute category, as answer options. We map the
object annotation $b$ to the visual-token grid as
in Section~\ref{subsec:method_overview},
obtaining the inside and outside token sets $\mathcal K_{\mathrm{in}}$ and
$\mathcal K_{\mathrm{out}}$. We keep the model frozen and extract its
attention weights from a single forward pass.
Figure~\ref{fig:attention_baseline_example} illustrates the attention-ratio
computation for an object and a candidate attribute.

\begin{figure}[!htbp]
    \centering
    \includegraphics[width=\linewidth]{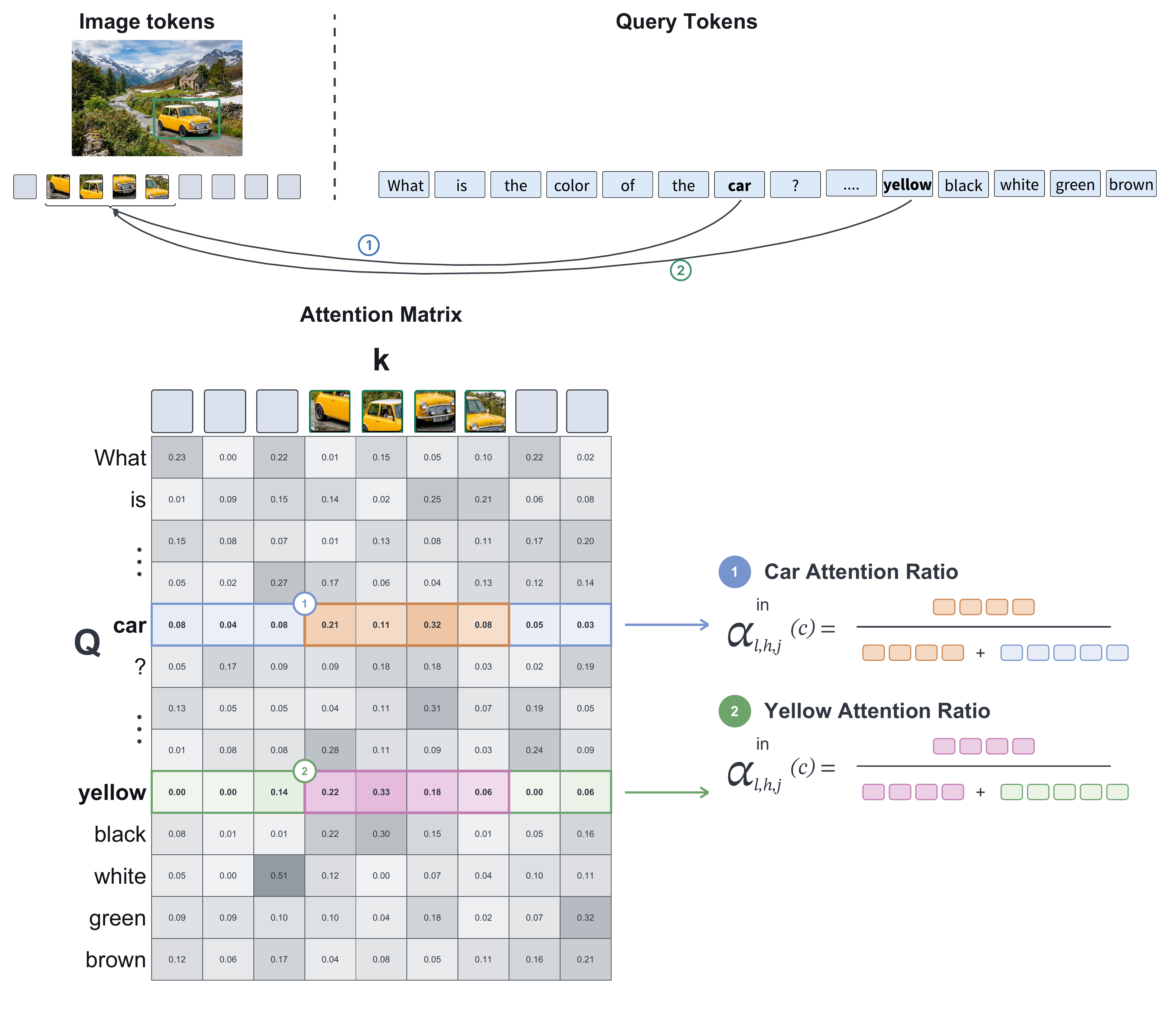}
    \caption{\textbf{Computing object and attribute attention ratios.}
    \emph{Top:} We identify the visual tokens inside the annotated object
    region and the tokens in query $Q$ representing the target object
    (``car'', 1) and a candidate attribute (``yellow'', 2).
    \emph{Bottom:} The highlighted rows show attention from these concept
    tokens to visual tokens at one layer $l$ and head $h$. For each concept,
    attention inside the object region is summed and divided by total image
    attention to obtain the inside attention ratio
    $\alpha_{l,h,j}^{\mathrm{in}}(c)$ (Eq.~\ref{eq:attention_region_ratios}).}
    \label{fig:attention_baseline_example}
\end{figure}

\paragraph{Attention-Evidence Extraction.}
Let $c$ be either a candidate attribute $a\in\mathcal A$ or the target object
$o$, and let $q_1,\ldots,q_m$ be the tokens representing $c$ in the query
$Q$. For layer
$l\in\{1,\ldots,L\}$, head
$h\in\{1,\ldots,H\}$, and concept token $q_j$, let $A_{l,h}(q_j,k_i)$ denote
the attention from $q_j$ to visual token $k_i$. We first compute the inside and
outside attention ratios, defined as the fractions of a concept token's image
attention assigned inside and outside the target-object region:
\begin{equation}
\begin{aligned}
\alpha_{l,h,j}^{\mathrm{in}}(c)
&:=\frac{\sum_{k_i\in\mathcal K_{\mathrm{in}}}A_{l,h}(q_j,k_i)}
{\sum_{k_i\in\mathcal K_{\mathrm{in}}\cup\mathcal K_{\mathrm{out}}}
A_{l,h}(q_j,k_i)},\\
\alpha_{l,h,j}^{\mathrm{out}}(c)
&:=\frac{\sum_{k_i\in\mathcal K_{\mathrm{out}}}A_{l,h}(q_j,k_i)}
{\sum_{k_i\in\mathcal K_{\mathrm{in}}\cup\mathcal K_{\mathrm{out}}}
A_{l,h}(q_j,k_i)}.
\end{aligned}
\label{eq:attention_region_ratios}
\end{equation}
Thus,
$\alpha_{l,h,j}^{\mathrm{in}}(c)+
\alpha_{l,h,j}^{\mathrm{out}}(c)=1$.
Because the regions may contain different numbers of visual tokens, we divide
each attention ratio by its region size, giving the inside and outside attention
densities:
\begin{equation}
\begin{aligned}
d_{l,h,j}^{\mathrm{in}}(c)
&:=\frac{\alpha_{l,h,j}^{\mathrm{in}}(c)}
{|\mathcal K_{\mathrm{in}}|}, &
d_{l,h,j}^{\mathrm{out}}(c)
&:=\frac{\alpha_{l,h,j}^{\mathrm{out}}(c)}
{|\mathcal K_{\mathrm{out}}|}.
\end{aligned}
\label{eq:attention_region_densities}
\end{equation}
Under uniform attention over image patches,
$d_{l,h,j}^{\mathrm{in}}(c)=d_{l,h,j}^{\mathrm{out}}(c)$. We next explain how
these densities are turned into the probes used in
Section~\ref{subsec:probe_supervision}.

\paragraph{Attribute-Evidence Regression.}
For a candidate attribute $a$, we compare its inside and outside densities:
\begin{equation}
p_{l,h,j}^{\mathrm{in}}(a)
:=\frac{d_{l,h,j}^{\mathrm{in}}(a)}
{d_{l,h,j}^{\mathrm{in}}(a)+d_{l,h,j}^{\mathrm{out}}(a)}.
\label{eq:attribute_attention_score}
\end{equation}
Uniform image attention gives $p_{l,h,j}^{\mathrm{in}}(a)=1/2$, whereas a
larger value indicates preferential attention to the target-object region.
At each layer, we average over the $m$ concept tokens and aggregate heads using
one of the strategies defined in \S\ref{app:aggregation}:
\begin{equation}
S_{l,\mathrm{att}}^{\mathrm{attr}}(a)
:=\operatorname{Agg}_{h=1}^{H}
\left(\frac{1}{m}\sum_{j=1}^{m}p_{l,h,j}^{\mathrm{in}}(a)\right).
\label{eq:attribute_attention_aggregation}
\end{equation}
Stacking the layer scores yields
$\mathbf S_{\mathrm{att}}^{\mathrm{attr}}(a)
:=[S_{1,\mathrm{att}}^{\mathrm{attr}}(a),\ldots,
S_{L,\mathrm{att}}^{\mathrm{attr}}(a)]^\top$, labeled $z=1$ when $a$ is
annotated as present on the target object and $z=0$ otherwise.

\paragraph{Object-Evidence Regression.}
For the target object $o$, we retain the two densities as separate features.
Averaging over its $m$ concept tokens and aggregating heads as above gives
\begin{equation}
\begin{aligned}
S_{l,\mathrm{att}}^{\mathrm{obj,in}}(o)
&:=\operatorname{Agg}_{h=1}^{H}
\left(\frac{1}{m}\sum_{j=1}^{m}d_{l,h,j}^{\mathrm{in}}(o)\right),\\
S_{l,\mathrm{att}}^{\mathrm{obj,out}}(o)
&:=\operatorname{Agg}_{h=1}^{H}
\left(\frac{1}{m}\sum_{j=1}^{m}d_{l,h,j}^{\mathrm{out}}(o)\right).
\end{aligned}
\label{eq:object_attention_aggregation}
\end{equation}
Stacking the layer scores gives
$\mathbf S_{\mathrm{att}}^{\mathrm{obj,in}}(o)$ and
$\mathbf S_{\mathrm{att}}^{\mathrm{obj,out}}(o)$. The inside vector is a
positive example ($z=1$), whereas the outside vector is a negative example
($z=0$).

\FloatBarrier
\subsection{Logit Lens}
\label{app_sub:logit_lens}

The logit lens \citep{nostalgebraist2020logitlens} provides a window into the interpretability
of intermediate representations, by applying the output head at every layer rather than only
at the last. In VLMs, it relates visual tokens that encode specific visual information to
vocabulary concepts.

We apply it to \emph{visual} tokens: for the representation $h_l(k_i)$ of
visual token $k_i$ at layer $l$, the head $W_V\in\mathbb R^{|V|\times d}$ yields logits over
the vocabulary $V$, and the projected rank of each $v\in V$ is
\begin{equation}
\operatorname{rank}_l(k_i,v)
:=1+\bigl|\{\,u\in V:\ [W_V h_l(k_i)]_u>[W_V h_l(k_i)]_v\,\}\bigr|,
\label{eq:logit_lens_rank}
\end{equation}
running from one (highest scoring) to $|V|$ (lowest scoring). Because causal masking prevents
visual-token representations from attending to the question $Q$, these rankings depend only on
the image $I$.

Figure~\ref{fig:app_logit_lens} shows this projection for visual tokens sampled inside annotated
VAW object regions. The depth at which object and attribute semantics emerge varies across
objects, consistent with prior work placing semantic interpretation in the middle-to-late
layers \citep{DBLP:conf/iclr/NeoO0G0B25, DBLP:conf/iclr/JiangKPG25,
DBLP:conf/cvpr/JiangCZLSY25}. For the plate, for example, the sampled in-region tokens are
uninformative in the early layers and decode to \emph{plate}, \emph{white} and \emph{round}
only from the middle layers onward, covering both the object identity and its attributes, the
two signals \method{} probes.

\begin{figure}[htbp]
    \centering
    \includegraphics[width=\linewidth]{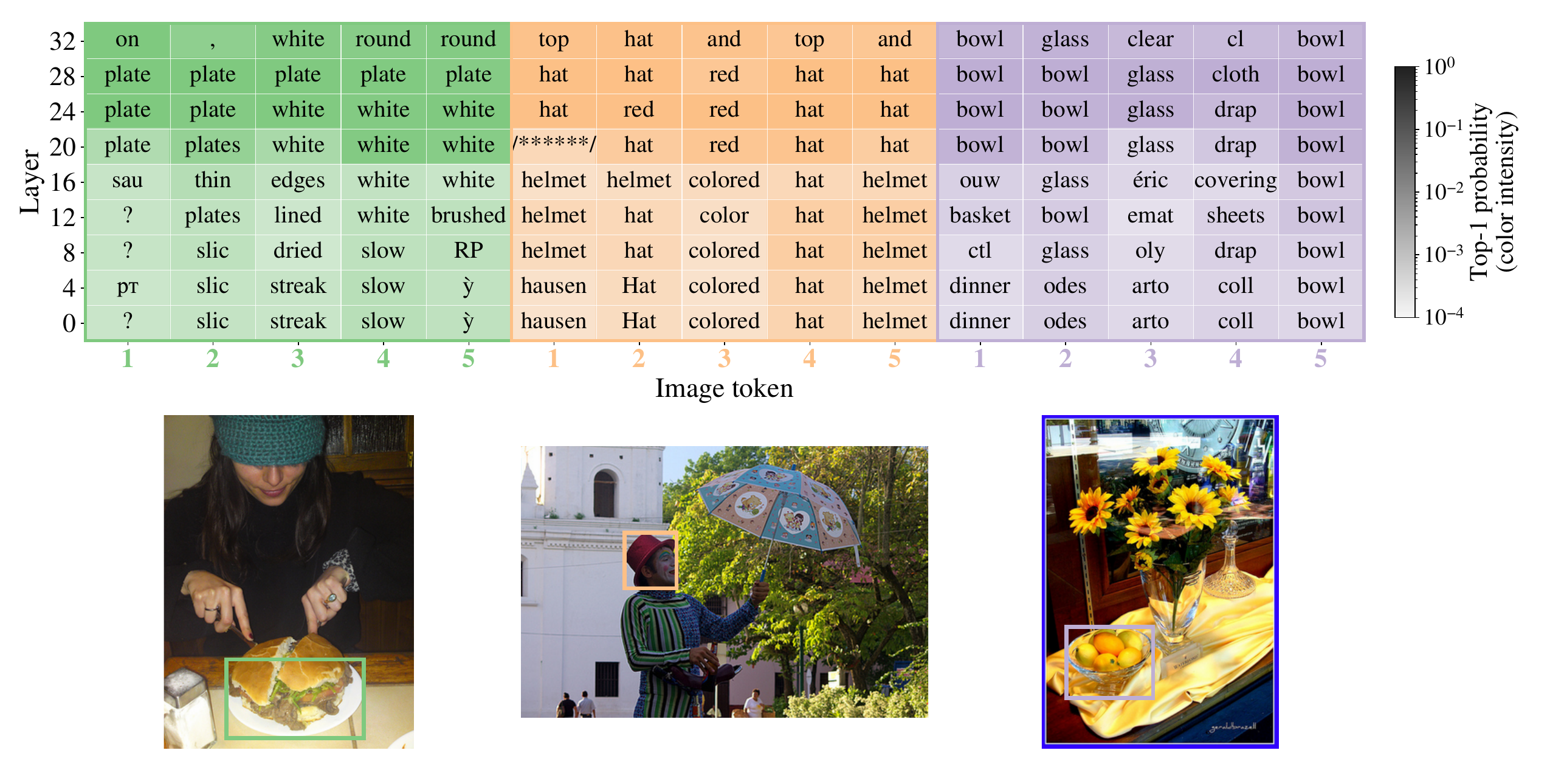}
    \caption{\textbf{Logit-lens projection of visual tokens inside annotated object regions.}
    Qualitative examples from VAW in LLaVA, for three target objects: a plate (green), a hat
    (orange), and a bowl (purple).
    \emph{Top:} For each object, we randomly sample five visual tokens $k_i$ from
    its inside-region set $\mathcal K_{\mathrm{in}}$ and show, at every layer $l$, the
    top-ranked vocabulary token, i.e.\ the $v\in V$ with $\operatorname{rank}_l(k_i,v)=1$
    (Eq.~\ref{eq:logit_lens_rank}). Color intensity indicates the top-1 probability.
    \emph{Bottom:} The corresponding images, each with the annotated bounding box $b$ of its
    target object in the matching color.
    Semantics emerge at different depths across the three objects: the plate tokens are
    uninformative until the middle layers and then decode to \emph{plate}, \emph{white} and
    \emph{round}, whereas \emph{hat} and \emph{bowl} are already top-ranked in the earliest
    layers, with \emph{red} and \emph{glass} following later.}
    \label{fig:app_logit_lens}
\end{figure}

\subsection{Dataset Preprocessing}
\label{app_subsec:data preprocessing}

\paragraph{Filtering and normalization.}
We preprocess the dataset to ensure unambiguous supervision and improve compatibility with our vocabulary-based analysis. Specifically, (1) we remove object instances containing multiple present attributes within the same category, since such cases do not define a unique target attribute; (2) we filter out images containing multiple instances of the same object category with different attribute annotations, as this may introduce ambiguity regarding which instance the model refers to; and (3) we normalize fine-grained attribute values into canonical forms (e.g., \textit{light brown} and \textit{dark brown} $\rightarrow$ \textit{brown}) to reduce label sparsity and improve alignment with vocabulary-level representations, yielding 11 colors, 15 materials, and 5 shapes, with the full mapping given in Table~\ref{tab:app_canonical_map}.

\paragraph{Question construction.}
\label{app_par:question_construction}
Each VAW instance provides an object name $o$, a region (polygon or bounding box), and a set of present and absent attributes. From every instance we construct one question per attribute category $a\in\{\texttt{color},\texttt{material},\texttt{shape}\}$ for which the instance carries a ground-truth value. We use the object name exactly as it appears in the annotation. The region is not shown to the model and only serves to select the visual tokens that \method{} uses for evidence detection (\S\ref{subsec:method_overview}). All questions follow a single template, shared by both models (Figure~\ref{fig:app_prompt}).

\begin{figure}[htbp]
\centering
\fbox{\begin{minipage}{0.92\linewidth}
\small
\emph{Template.}\\[2pt]
{\ttfamily What is the \{$a$\} of the \{$o$\}? Answer with exactly ONE of the
following options, and output only the option text: $v_1$, $v_2$, \ldots, $v_k$.}\\[6pt]
\emph{Example} (category \texttt{color}, object \texttt{traffic light}).\\[2pt]
{\ttfamily What is the color of the traffic light? Answer with exactly ONE of the
following options, and output only the option text: purple, white, red, brown,
yellow, black, green, gray, pink, blue, orange.}\\[6pt]
\emph{Chat format}, with the question above as \texttt{\{q\}}.\\[2pt]
LLaVA:\\
{\ttfamily [INST] \textless image\textgreater{} \{q\} [/INST]}\\[2pt]
Qwen:\\
{\ttfamily \textless |im\_start|\textgreater system}\\
{\ttfamily You are a helpful assistant.\textless |im\_end|\textgreater}\\
{\ttfamily \textless |im\_start|\textgreater user}\\
{\ttfamily \textless |vision\_start|\textgreater\textless |image\_pad|\textgreater
\textless |vision\_end|\textgreater\{q\}\textless |im\_end|\textgreater}\\
{\ttfamily \textless |im\_start|\textgreater assistant}
\end{minipage}}
\caption{\textbf{Question template, an example question, and the per-model chat formats.}
The template is filled with the attribute category $a$, the object name $o$, and the canonical
option set $\mathcal V_a$ in a randomly permuted order. The filled question, denoted
\texttt{\{q\}}, is then wrapped in each model's chat format.}
\label{fig:app_prompt}
\end{figure}

Here $\mathcal V_a=\{v_1,\ldots,v_k\}$ is the canonical option set for category $a$ (Table~\ref{tab:app_option_sets}), whose order is randomly permuted per instance so that the correct answer is not tied to a fixed position. Each permutation is drawn once and then cached: every subsequent run reuses the identical stored prompt, including the controlled-transformation conditions of \S\ref{app:controlled_transformations}, so that differences in behavior across conditions cannot be attributed to a change in option order. The filled template is then wrapped in each model's own chat format.

\begin{table}[t]
\centering
\small
\caption{\textbf{Canonical option sets.} The options presented to the model for each attribute category.}
\label{tab:app_option_sets}
\begin{adjustbox}{width=0.98\textwidth}
\begin{tabular}{@{}p{0.12\linewidth}p{0.03\linewidth}p{\dimexpr0.85\linewidth-4\tabcolsep\relax}@{}}
\toprule
Category & $k$ & Options \\
\midrule
Color & 11 & black, blue, brown, gray, green, orange, pink, purple, red, white, yellow \\
Material & 15 & wood, metal, glass, plastic, rubber, foam, paper, fabric, leather, fur, ceramic, stone, dirt, snow, water \\
Shape & 5 & round, square, curved, triangular, curly \\
\bottomrule
\end{tabular}
\end{adjustbox}
\end{table}

\begin{table}[t]
\centering
\small
\caption{\textbf{Attribute-value normalization.} Each canonical value is listed with the full
set of raw VAW attribute strings mapped onto it.}
\label{tab:app_canonical_map}
\begin{adjustbox}{width=0.98\textwidth}
\begin{tabular}{@{}p{0.18\linewidth}p{\dimexpr0.82\linewidth-2\tabcolsep\relax}@{}}
\toprule
Canonical value & VAW values mapped onto it \\
\midrule
\multicolumn{2}{@{}l}{\emph{Color} (61 raw $\rightarrow$ 11 canonical)}\\
black & black \\
blue & aqua, blue, bluish, bright blue, dark blue, light blue, navy blue, pale blue, teal, turquoise \\
brown & bronze, brown, dark brown, golden brown, light brown, tan, taupe \\
gray & dark gray, gray, light gray, murky, off white, silver \\
green & bright green, dark green, green, green tinted, light green, lime green, neon green, olive green, pale green \\
orange & orange, rust colored, terracotta \\
pink & dark pink, fuchsia, pink \\
purple & dark purple, purple, purplish, violet \\
red & bright red, burgundy, dark red, light red, maroon, red, reddish \\
white & beige, cream colored, ivory, white \\
yellow & amber, blond, bright yellow, golden, light yellow, sand colored, yellow \\
\midrule
\multicolumn{2}{@{}l}{\emph{Material} (50 raw $\rightarrow$ 15 canonical)}\\
wood & bamboo, hardwood, wicker, wooden \\
metal & aluminum, brass, chrome, iron, metal, stainless steel, steel \\
glass & glass \\
plastic & plastic \\
rubber & rubber \\
foam & foamy, styrofoam \\
paper & cardboard, paper \\
fabric & cloth, denim, fabric, jean, khaki, silk \\
leather & leather \\
fur & fluffy, furry, plush \\
ceramic & ceramic, clay, porcelain, tiled \\
stone & asphalt, brick, cement, cobblestone, concrete, granite, gravel, marble, pebbled, red brick, rocky, stone, stucco \\
dirt & dirt, muddy, sandy \\
snow & snowy \\
water & water \\
\midrule
\multicolumn{2}{@{}l}{\emph{Shape} (20 raw $\rightarrow$ 5 canonical)}\\
round & circular, domed, oval shaped, round, rounded, spherical \\
square & cubed, rectangular, square \\
curved & arch shaped, curved, cylindrical \\
triangular & arrow shaped, diamond shaped, octagonal, pointy, spiky, triangular \\
curly & curled, curly \\
\bottomrule
\end{tabular}
\end{adjustbox}
\end{table}

\paragraph{Answer generation.}
We keep the VLM frozen and decode greedily (\texttt{do\_sample=False}, \texttt{num\_beams=1}, \texttt{max\_new\_tokens=10}), so answers are reproducible across runs. The answer is generated rather than scored against the candidates, and the option list appears in the prompt only to constrain the output to a known vocabulary, so that predictions can be compared against VAW labels without a free-text matching step.

\section{Complementary Results}
\label{app:detailed_results}

\subsection{Development-Set Configuration Selection}
\label{app_sub:dev_set_results}

\paragraph{Development Split.}
To choose the configuration of our rank-based method and of each baseline, we randomly divided the
original training split into 60\% for fitting the candidate estimators and 40\% for development-set
evaluation. Only these training-derived subsets were used for configuration selection, leaving the
test set untouched. Tables~\ref{tab:attribute_prediction_detailed}
and~\ref{tab:object_existence_detailed} report development-set performance for every configuration
considered, on LLaVA and Qwen and separately for color, material, and shape. AUC is threshold-free,
whereas F1 and accuracy are computed at a probability threshold of $0.5$, at which a probe declares
the evidence present.

\paragraph{Attribute-Evidence Configurations.}
Following Section~\ref{sec:experimental}, we varied two design choices: the \emph{aggregation
strategy} used to summarize visual-token signals (\S\ref{app:aggregation}), and the \emph{pooling configuration}, which
determines whether one estimator is shared across all values in a category (\emph{All}) or one is
fitted per value (\emph{Separate}). This yields 18 configurations per model and attribute category (8 rank-based, 2 confidence-based, 2 classifier-based, and 6 attention-based).
For the rank-based method we crossed both pooling configurations with the minimum, maximum, mean,
and median aggregation strategies. The confidence-based baseline was evaluated with both pooling
configurations but requires no aggregation strategy, as it retains a single maximum confidence
value. The classifier-based baseline is inherently \emph{Separate}, and we tested the mean and
maximum strategies. For the attention-based baseline we crossed both pooling configurations with
the mean, median, and maximum strategies.

\paragraph{Object-Evidence Configurations.}
For object evidence a single estimator is shared across object categories, so the distinction between
\emph{All} and \emph{Separate} does not apply. We evaluated eight configurations: four rank
aggregation strategies, the confidence baseline without an aggregation strategy, and three
attention aggregation strategies. We do not evaluate the classifier baseline for objects, because
its fixed label space is incompatible with the open-ended object vocabulary.

\paragraph{Selected Configurations.}
We selected one aggregation strategy per method, shared across attribute categories and across both
VLMs, favoring the strategy that was strongest across AUC, F1, and accuracy. For attribute evidence this gives the mean strategy for the
rank- and classifier-based methods and the maximum strategy for the attention-based baseline, and for
object evidence the median strategy for the rank- and attention-based methods. Pooling is treated
differently: we keep both \emph{All} and \emph{Separate} for the rank-, confidence-, and
attention-based estimators, as they are two distinct estimators, one shared across the values of a
category and one fitted per value. After fixing these
choices we refit each estimator on the complete training split, the 60\% fitting and 40\% development
portions combined, and evaluated it once on the held-out test set, giving the results reported in
Table~\ref{tab:evidence_detection_auc}.

\paragraph{Differences Across Attribute Categories.}
The rank-based results in Table~\ref{tab:attribute_prediction_detailed} follow the same ordering in both
models: color and material yield comparably strong evidence, whereas shape is consistently the most
challenging. A plausible explanation is the amount of visual context each property requires, even for a
human observer. A small patch near the center of a red wooden table may be enough to identify its color,
and the grain visible on that same surface likewise reveals its material. Deciding whether the tabletop is round or rectangular, by contrast, requires seeing its outer
boundary, so no local region is sufficient. Shape is therefore the property least likely to be recoverable
from the tokens inside the annotated object region.

\begin{table}[t]
\centering
\caption{\textbf{Development-set attribute-evidence prediction.} AUC, F1, and
accuracy for the attribute probe $M_{\mathrm{attr}}$ under every combination of
pooling configuration (\emph{All}, \emph{Separate}) and aggregation strategy
(\S\ref{app:aggregation}), reported for LLaVA and Qwen and grouped by attribute
category. Bold marks the best value in each column within a method and pooling
configuration.}
\label{tab:attribute_prediction_detailed}
\fontsize{6.8}{7.4}\selectfont
\renewcommand{\arraystretch}{0.82}
\begin{adjustbox}{width=0.98\textwidth}
\begin{tabular}{lll|ccc|ccc}
\toprule
\cellcolor{Gray!15}
& \cellcolor{Gray!15}
& \cellcolor{Gray!15}
& \multicolumn{3}{c|}{\cellcolor{Gray!40}\textbf{LLaVA}}
& \multicolumn{3}{c}{\cellcolor{Gray!40}\textbf{Qwen}} \\
\cellcolor{Gray!15} Category
& \cellcolor{Gray!15} Method
& \cellcolor{Gray!15} Aggregation
& \cellcolor{Gray!15} AUC
& \cellcolor{Gray!15} F1
& \cellcolor{Gray!15} Acc.
& \cellcolor{Gray!15} AUC
& \cellcolor{Gray!15} F1
& \cellcolor{Gray!15} Acc. \\
\midrule
\rowcolor{Gray!15}
\multicolumn{9}{c}{\textbf{Attribute: Color}\quad
{\tiny $N_{\mathrm{LLaVA}}=N_{\mathrm{Qwen}}=768$}} \\
\midrule
Rank-based & All & Min & 0.65 & 0.54 & 0.62 & 0.70 & 0.61 & 0.66 \\
 & All & Max & 0.75 & 0.75 & 0.70 & 0.76 & 0.74 & 0.70 \\
 & All & Mean & \textbf{0.95} & \textbf{0.89} & \textbf{0.88} & \textbf{0.91} & \textbf{0.82} & \textbf{0.82} \\
 & All & Median & 0.94 & 0.88 & \textbf{0.88} & 0.90 & \textbf{0.82} & 0.81 \\
 & Separate & Min & 0.55 & 0.66 & 0.53 & 0.70 & 0.62 & 0.66 \\
 & Separate & Max & 0.56 & 0.68 & 0.53 & 0.75 & 0.75 & 0.72 \\
 & Separate & Mean & \textbf{0.96} & 0.89 & \textbf{0.89} & \textbf{0.92} & \textbf{0.83} & \textbf{0.83} \\
 & Separate & Median & \textbf{0.96} & \textbf{0.90} & \textbf{0.89} & 0.90 & 0.81 & 0.81 \\
\cmidrule(lr){1-9}
Confidence-based & All & -- & 0.79 & 0.73 & 0.70 & 0.79 & 0.72 & 0.71 \\
 & Separate & -- & 0.81 & 0.76 & 0.73 & 0.79 & 0.73 & 0.72 \\
\cmidrule(lr){1-9}
Classifier-based & Separate & Mean & 0.92 & 0.85 & 0.85 & 0.92 & \textbf{0.84} & \textbf{0.84} \\
 & Separate & Max & 0.90 & 0.82 & 0.82 & 0.88 & 0.78 & 0.79 \\
\cmidrule(lr){1-9}
Attention-based & All & Mean & 0.78 & 0.71 & \textbf{0.72} & \textbf{0.80} & \textbf{0.73} & \textbf{0.73} \\
 & All & Median & 0.76 & 0.70 & 0.71 & 0.75 & 0.71 & 0.70 \\
 & All & Max & \textbf{0.79} & \textbf{0.72} & \textbf{0.72} & 0.79 & \textbf{0.73} & 0.72 \\
 & Separate & Mean & 0.76 & 0.70 & 0.69 & 0.77 & 0.70 & 0.70 \\
 & Separate & Median & 0.75 & 0.69 & 0.69 & 0.74 & 0.68 & 0.68 \\
 & Separate & Max & \textbf{0.78} & \textbf{0.72} & \textbf{0.71} & \textbf{0.79} & \textbf{0.73} & \textbf{0.72} \\
\midrule
\rowcolor{Gray!15}
\multicolumn{9}{c}{\textbf{Attribute: Material}\quad
{\tiny $N_{\mathrm{LLaVA}}=N_{\mathrm{Qwen}}=632$}} \\
\midrule
Rank-based & All & Min & 0.62 & 0.54 & 0.59 & 0.64 & 0.54 & 0.59 \\
 & All & Max & 0.75 & 0.76 & 0.71 & 0.81 & 0.78 & 0.76 \\
 & All & Mean & \textbf{0.90} & \textbf{0.81} & \textbf{0.81} & \textbf{0.86} & \textbf{0.80} & \textbf{0.80} \\
 & All & Median & 0.89 & \textbf{0.81} & \textbf{0.81} & \textbf{0.86} & 0.79 & 0.79 \\
 & Separate & Min & 0.77 & 0.68 & 0.71 & 0.67 & 0.59 & 0.62 \\
 & Separate & Max & 0.83 & 0.78 & 0.76 & 0.86 & 0.78 & 0.77 \\
 & Separate & Mean & \textbf{0.97} & \textbf{0.91} & \textbf{0.91} & \textbf{0.95} & \textbf{0.86} & \textbf{0.87} \\
 & Separate & Median & 0.96 & 0.90 & 0.90 & 0.94 & 0.85 & 0.85 \\
\cmidrule(lr){1-9}
Confidence-based & All & -- & 0.85 & 0.77 & 0.77 & 0.83 & 0.75 & 0.76 \\
 & Separate & -- & 0.87 & 0.78 & 0.78 & 0.85 & 0.76 & 0.77 \\
\cmidrule(lr){1-9}
Classifier-based & Separate & Mean & 0.92 & 0.85 & 0.84 & 0.92 & 0.86 & 0.85 \\
 & Separate & Max & 0.92 & 0.85 & 0.85 & 0.92 & 0.86 & 0.86 \\
\cmidrule(lr){1-9}
Attention-based & All & Mean & 0.70 & \textbf{0.65} & 0.64 & 0.74 & 0.69 & 0.67 \\
 & All & Median & 0.70 & \textbf{0.65} & \textbf{0.65} & 0.71 & 0.69 & 0.67 \\
 & All & Max & \textbf{0.71} & \textbf{0.65} & 0.64 & \textbf{0.77} & \textbf{0.72} & \textbf{0.70} \\
 & Separate & Mean & 0.72 & 0.66 & \textbf{0.67} & 0.73 & 0.69 & 0.68 \\
 & Separate & Median & 0.71 & 0.66 & 0.65 & 0.71 & 0.65 & 0.64 \\
 & Separate & Max & \textbf{0.74} & \textbf{0.68} & \textbf{0.67} & \textbf{0.75} & \textbf{0.70} & \textbf{0.69} \\
\midrule
\rowcolor{Gray!15}
\multicolumn{9}{c}{\textbf{Attribute: Shape}\quad
{\tiny $N_{\mathrm{LLaVA}}=N_{\mathrm{Qwen}}=480$}} \\
\midrule
Rank-based & All & Min & 0.61 & 0.55 & 0.59 & 0.56 & 0.44 & 0.54 \\
 & All & Max & 0.53 & 0.67 & 0.54 & 0.72 & 0.69 & 0.66 \\
 & All & Mean & \textbf{0.84} & \textbf{0.76} & \textbf{0.75} & \textbf{0.83} & \textbf{0.74} & \textbf{0.74} \\
 & All & Median & 0.78 & 0.73 & 0.72 & 0.77 & 0.70 & 0.70 \\
 & Separate & Min & 0.66 & 0.64 & 0.62 & 0.57 & 0.49 & 0.57 \\
 & Separate & Max & 0.55 & 0.64 & 0.54 & 0.73 & 0.66 & 0.65 \\
 & Separate & Mean & \textbf{0.90} & \textbf{0.82} & \textbf{0.81} & \textbf{0.88} & \textbf{0.78} & \textbf{0.78} \\
 & Separate & Median & 0.85 & 0.78 & 0.78 & 0.82 & 0.72 & 0.72 \\
\cmidrule(lr){1-9}
Confidence-based & All & -- & 0.70 & 0.62 & 0.64 & 0.72 & 0.64 & 0.66 \\
 & Separate & -- & 0.72 & 0.63 & 0.66 & 0.74 & 0.66 & 0.67 \\
\cmidrule(lr){1-9}
Classifier-based & Separate & Mean & 0.82 & 0.74 & 0.74 & 0.84 & 0.77 & 0.77 \\
 & Separate & Max & 0.80 & 0.75 & 0.75 & 0.83 & 0.76 & 0.75 \\
\cmidrule(lr){1-9}
Attention-based & All & Mean & 0.64 & 0.62 & 0.61 & \textbf{0.65} & 0.61 & 0.59 \\
 & All & Median & 0.65 & \textbf{0.63} & \textbf{0.62} & \textbf{0.65} & 0.61 & \textbf{0.61} \\
 & All & Max & \textbf{0.67} & \textbf{0.63} & 0.60 & 0.63 & \textbf{0.62} & 0.59 \\
 & Separate & Mean & 0.67 & 0.62 & 0.61 & \textbf{0.66} & 0.61 & 0.60 \\
 & Separate & Median & 0.67 & 0.63 & \textbf{0.62} & \textbf{0.66} & 0.61 & 0.60 \\
 & Separate & Max & \textbf{0.70} & \textbf{0.65} & \textbf{0.62} & \textbf{0.66} & \textbf{0.64} & \textbf{0.61} \\
\bottomrule
\end{tabular}
\end{adjustbox}
\vspace{-0.8em}
\end{table}

\begin{table}[t]
\centering
\caption{\textbf{Development-set object-evidence prediction.} AUC, F1, and accuracy
for the object probe $M_{\mathrm{obj}}$ under each aggregation strategy
(\S\ref{app:aggregation}), reported for LLaVA and Qwen and grouped by the attribute
category of the query. Pooling does not apply here, as a single estimator is shared
across object categories. Bold marks the best value in each column within a method.}
\label{tab:object_existence_detailed}
\small
\renewcommand{\arraystretch}{1.02}
\begin{adjustbox}{width=0.98\textwidth}
\begin{tabular}{ll|ccc|ccc}
\toprule
\cellcolor{Gray!15}
& \cellcolor{Gray!15}
& \multicolumn{3}{c|}{\cellcolor{Gray!40}\textbf{LLaVA}}
& \multicolumn{3}{c}{\cellcolor{Gray!40}\textbf{Qwen}} \\
\cellcolor{Gray!15} Category
& \cellcolor{Gray!15} Aggregation
& \cellcolor{Gray!15} AUC
& \cellcolor{Gray!15} F1
& \cellcolor{Gray!15} Acc.
& \cellcolor{Gray!15} AUC
& \cellcolor{Gray!15} F1
& \cellcolor{Gray!15} Acc. \\
\midrule
\rowcolor{Gray!15}
\multicolumn{8}{c}{\textbf{Attribute: Color}\quad
{\tiny $N_{\mathrm{LLaVA}}=N_{\mathrm{Qwen}}=1{,}494$}} \\
\midrule
Rank-based & Min & 0.52 & 0.33 & 0.53 & 0.75 & 0.65 & 0.70 \\
 & Max & 0.87 & 0.75 & 0.77 & \textbf{0.90} & \textbf{0.82} & \textbf{0.82} \\
 & Mean & 0.88 & 0.80 & 0.80 & 0.81 & 0.74 & 0.75 \\
 & Median & \textbf{0.91} & \textbf{0.83} & \textbf{0.82} & 0.83 & 0.76 & 0.75 \\
\cmidrule(lr){1-8}
Confidence-based & -- & 0.54 & 0.60 & 0.52 & 0.56 & 0.60 & 0.54 \\
\cmidrule(lr){1-8}
Attention-based & Mean & 1.00 & 0.98 & 0.98 & 0.96 & 0.93 & 0.94 \\
 & Median & \textbf{1.00} & \textbf{0.99} & \textbf{0.99} & 0.96 & \textbf{0.93} & \textbf{0.94} \\
 & Max & 1.00 & 0.97 & 0.97 & \textbf{0.97} & 0.93 & 0.93 \\
\midrule
\rowcolor{Gray!15}
\multicolumn{8}{c}{\textbf{Attribute: Material}\quad
{\tiny $N_{\mathrm{LLaVA}}=N_{\mathrm{Qwen}}=1{,}224$}} \\
\midrule
Rank-based & Min & 0.51 & 0.29 & 0.53 & 0.74 & 0.65 & 0.70 \\
 & Max & 0.86 & 0.75 & 0.77 & \textbf{0.89} & \textbf{0.80} & \textbf{0.81} \\
 & Mean & 0.89 & 0.81 & 0.80 & 0.83 & 0.75 & 0.75 \\
 & Median & \textbf{0.91} & \textbf{0.84} & \textbf{0.83} & 0.85 & 0.78 & 0.78 \\
\cmidrule(lr){1-8}
Confidence-based & -- & 0.52 & 0.59 & 0.51 & 0.58 & 0.61 & 0.54 \\
\cmidrule(lr){1-8}
Attention-based & Mean & \textbf{1.00} & 0.98 & 0.98 & \textbf{0.97} & \textbf{0.94} & \textbf{0.94} \\
 & Median & 1.00 & \textbf{0.99} & \textbf{0.99} & 0.96 & 0.93 & 0.93 \\
 & Max & 0.99 & 0.95 & 0.96 & 0.96 & 0.92 & 0.92 \\
\midrule
\rowcolor{Gray!15}
\multicolumn{8}{c}{\textbf{Attribute: Shape}\quad
{\tiny $N_{\mathrm{LLaVA}}=N_{\mathrm{Qwen}}=940$}} \\
\midrule
Rank-based & Min & 0.51 & 0.34 & 0.54 & 0.76 & 0.67 & 0.71 \\
 & Max & 0.87 & 0.76 & 0.78 & \textbf{0.92} & \textbf{0.83} & \textbf{0.83} \\
 & Mean & 0.88 & 0.80 & 0.79 & 0.87 & 0.80 & 0.80 \\
 & Median & \textbf{0.91} & \textbf{0.83} & \textbf{0.81} & 0.87 & 0.81 & 0.80 \\
\cmidrule(lr){1-8}
Confidence-based & -- & 0.53 & 0.62 & 0.52 & 0.60 & 0.61 & 0.55 \\
\cmidrule(lr){1-8}
Attention-based & Mean & 1.00 & 0.98 & 0.99 & \textbf{0.98} & \textbf{0.95} & \textbf{0.95} \\
 & Median & \textbf{1.00} & \textbf{0.99} & \textbf{0.99} & 0.98 & 0.94 & 0.94 \\
 & Max & 0.99 & 0.96 & 0.96 & 0.97 & 0.92 & 0.93 \\
\bottomrule
\end{tabular}
\end{adjustbox}
\vspace{-0.8em}
\end{table}

\FloatBarrier
\subsection{Joint Object-and-Attribute AUC}
\label{app_sub:joint_auc}

The \emph{Attr$\times$Obj} column of Table~\ref{tab:evidence_detection_auc} evaluates the composite
probe $M=M_{\mathrm{obj}}\land M_{\mathrm{attr}}$
(\S\ref{subsec:visual_evidence_object_attribute}), which is correct only when the target object is
detected \emph{and} the attribute decision matches the annotation, whether the attribute is present
or absent. Reporting this as a single ROC-AUC requires one score per instance. Let
$p_{\mathrm{attr}}$ and $p_{\mathrm{obj}}$ denote the probabilities produced by the two probes before
thresholding (Eqs.~\ref{eq:attribute_probe} and~\ref{eq:object_probe}), with the object probe
evaluated on the inside region $\mathcal K_{\mathrm{in}}$, and define
\begin{equation}
s:=p_{\mathrm{obj}}\,(2p_{\mathrm{attr}}-1).
\label{eq:joint_score}
\end{equation}
The sign of $s$ carries the attribute decision and its magnitude the object evidence: positive when
the target value is supported, negative when it is rejected, and near zero when the object is not
localized. The two kinds of correct decision therefore fall at opposite ends of the ranking, while
instances lacking object evidence support neither. We report the area under the ROC curve obtained by ranking
instances by $s$, with $y=1$ when the attribute is annotated as present on the target object and
$y=0$ otherwise, pooled over color, material, and shape ($N=3{,}128$).

\FloatBarrier
\subsection{ROC Curves}
\label{app_sub:roc_curves}

Figure~\ref{fig:app_roc_curves} shows the curves behind the attribute AUCs in
Table~\ref{tab:evidence_detection_auc}. On LLaVA the rank-based probe dominates every baseline across
the whole range, so its advantage does not depend on the threshold chosen, whereas on Qwen the
classifier-based baseline is comparable to it.

\begin{figure}[htbp]
    \centering
    \includegraphics[width=\linewidth]{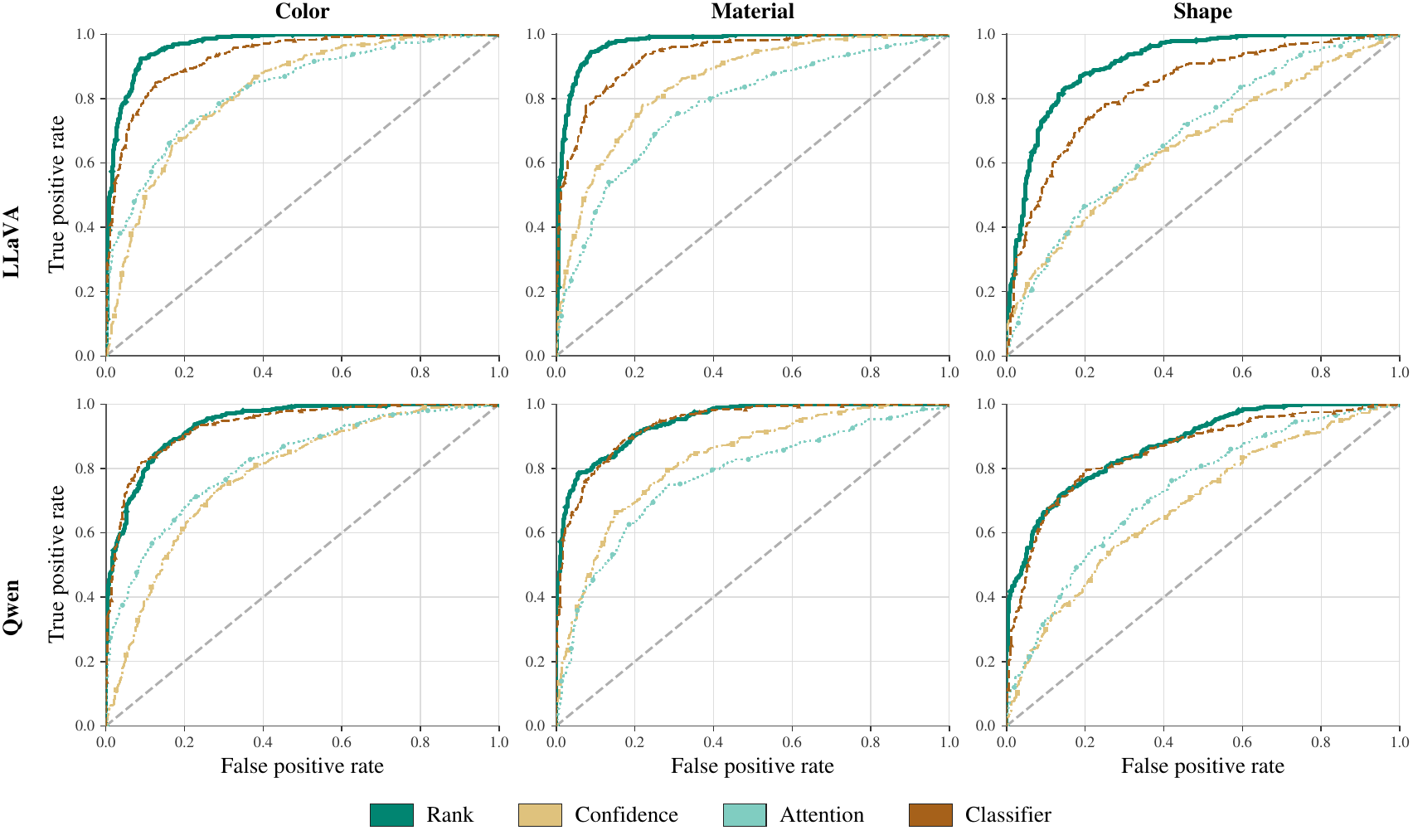}
    \caption{\textbf{ROC curves for attribute-evidence detection.} True-positive against false-positive rate
    for each method, shown for LLaVA (top) and Qwen (bottom) and separately for color, material, and shape,
    using the \emph{Separate} pooling configuration. The areas under these curves are the attribute columns
    of Table~\ref{tab:evidence_detection_auc}. The dashed diagonal marks chance.}
    \label{fig:app_roc_curves}
\end{figure}

\FloatBarrier
\subsection{Controlled Transformations}
\label{app:controlled_transformations}
\begin{figure}[htbp]
    \centering
    \includegraphics[width=\linewidth]{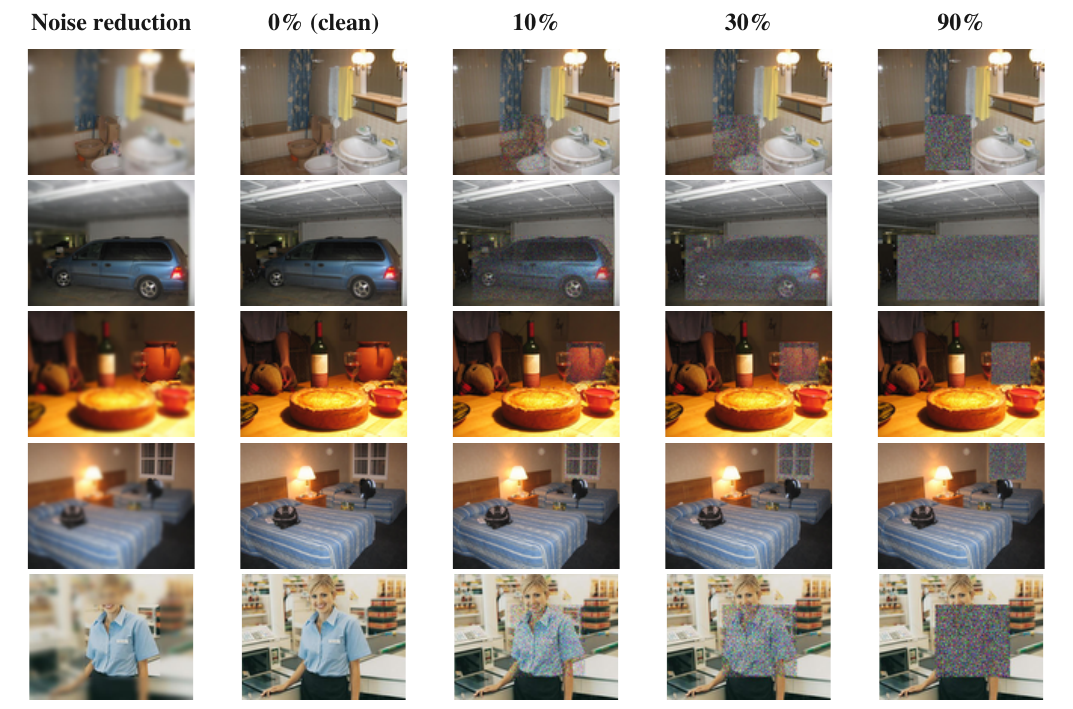}
    \caption{\textbf{Controlled transformations.} Five images from the VAW dataset (rows) with each transformation applied (columns): the object emphasized by blurring its surroundings ($\psi_+$), the clean image, and Gaussian noise within the bounding box $b$ at $\lambda=0.1$, $0.3$, and $0.9$ ($\psi_-$).}
    \label{fig:app_controlled_transformations}
\end{figure}

We instantiate the answer-preserving transformations $\psi_+$ and $\psi_-$ of Section~\ref{subsec:bounds_imperfect_probe} as pixel-level image edits, applied separately for each target-object instance. The positive transformation strengthens the relevant visual signal, while the negative transformation weakens it. Following Section~\ref{subsec:method_overview}, let $o$ denote the target object and $b$ its annotated bounding box, and let $I_p\in[0,1]^3$ be the normalized RGB value of pixel $p$ in $I$. Both transformations act only on the image component of the input $X=(I,Q)$, so that $X_+=(I_+,Q)$ and $X_-=(I_-,Q)$. Values are clipped to $[0,1]$ and converted back to 8-bit RGB.

\paragraph{Object Emphasis ($\psi_+$).}
The positive transformation highlights the relevant object, strengthening its visual signal by suppressing the surrounding context while leaving the queried object itself visually unchanged. We blur the full RGB image with a 2D Gaussian filter of kernel size $k=41$ and standard deviation $\sigma=k/6\approx6.83$, so that the kernel spans approximately $\pm3\sigma$ of the Gaussian, and composite the result with the original using the bounding box $b$ as a binary object mask:
\begin{equation}
(I_+)_p=
\begin{cases}
I_p, & p\in b,\\
(G_\sigma * I)_p, & p\notin b,
\end{cases}
\label{eq:controlled_blur}
\end{equation}
where $G_\sigma * I$ denotes the blurred image. The blur removes the edges, texture, and small structures of the background, leaving only its coarse layout and color, while the appearance of $o$ inside $b$ remains identical to $I$, so the ground-truth answer is preserved, i.e., $Y_{X_+}=Y_X$.

\paragraph{Object Degradation ($\psi_-$).}
The negative transformations corrupt only the pixels in the bounding box $b$. For each pixel and RGB channel, we independently sample $\varepsilon_p\sim\mathcal N(0,1)$ and compute
\begin{equation}
(I_-)_p=
\begin{cases}
\operatorname{clip}_{[0,1]}\!\left(\sqrt{1-\lambda}\,I_p+\sqrt{\lambda}\,\varepsilon_p\right), & p\in b,\\
I_p, & p\notin b.
\end{cases}
\label{eq:controlled_noise}
\end{equation}
As $\lambda$ grows, the corrupted region retains less of the original image and more Gaussian noise: at $\lambda=0.1$ the object region is only lightly perturbed, whereas at $\lambda=0.9$ the noise dominates and the original content is almost entirely replaced. We report the noise strength of this diffusion-style Gaussian corruption as $100\lambda\%$, so that 10\%, 30\%, and 90\% correspond to $\lambda=0.1$, $0.3$, and $0.9$, respectively. The strength $\lambda$ controls a trade-off: it must weaken the relevant visual signal, yet the transformation must remain answer-preserving. On the validation set we therefore select $\lambda=0.3$ as $\psi_-$, the strength at which the object region is visibly corrupted but the ground-truth answer is still unchanged, i.e., $Y_{X_-}=Y_X$. Figure~\ref{fig:app_controlled_transformations} illustrates this choice: the target object remains identifiable at $\lambda=0.3$, whereas at $\lambda=0.9$ it is largely destroyed and the correct answer is no longer recoverable. This is the transformation used to construct the bounds of Theorem~\ref{thm:vh_bounds}.

\subsection{Layer-Wise Visual-Evidence Development}
\label{app:layerwise_visual_evidence}

\begin{figure}[htbp]
    \centering
    \includegraphics[width=\linewidth]{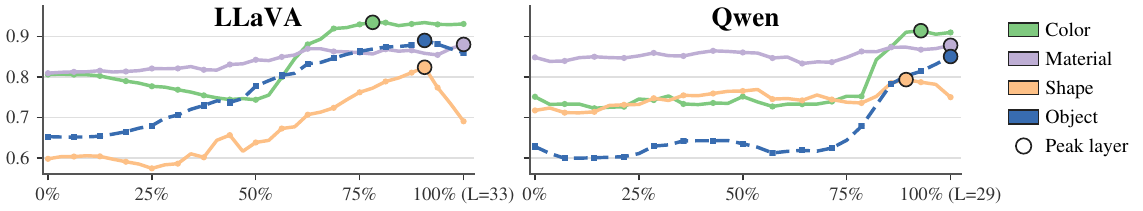}
    \caption{\textbf{Visual-evidence development across VLM layers.}
    Layer-wise attribute and object evidence for LLaVA (left) and Qwen (right). The figure reports the AUC obtained from each individual layer. The complete trajectory captures evidence that becomes accessible at different depths across attributes and models.}
    \label{fig:layerwise_visual_evidence}
\end{figure}

\FloatBarrier

\subsection{Attention Intervention}
\label{app_sub:attention-boost}

\paragraph{Intervention.}
Because evidence does not become accessible at a universal layer
(Figure~\ref{fig:layerwise_visual_evidence}), we intervene at every layer and head rather than at a
selected one. As in Section~\ref{app_sub:attention_baseline}, let $A_{l,h}(q_j,k_i)$ denote the
attention from query token $q_j\in Q$ to visual token $k_i$ at layer $l$ and head $h$, let
$z_{l,h}(q_j,\cdot)$ be the corresponding pre-softmax logits, and let $\mathcal K_{\mathrm{in}}$ be
the visual tokens inside the annotated target-object region $b$
(Section~\ref{subsec:method_overview}). For every query token we add a constant $\alpha$ to the
logits of the inside tokens before the softmax:
\begin{equation}
\tilde{A}_{l,h}(q_j,u)
=\frac{\exp\bigl(z_{l,h}(q_j,u)+\alpha\,\mathbb{I}\{u\in\mathcal K_{\mathrm{in}}\}\bigr)}
{\sum_{u'}\exp\bigl(z_{l,h}(q_j,u')+\alpha\,\mathbb{I}\{u'\in\mathcal K_{\mathrm{in}}\}\bigr)},
\label{eq:attention_intervention}
\end{equation}
where $u,u'$ range over all positions attended by $q_j$. Because the softmax normalizes over all
positions, this is equivalent to scaling the attention paid to the inside tokens by $e^{\alpha}$,
so attention elsewhere is reduced proportionally rather than removed. A single $\alpha$ is used for
every layer, head, and model.

\paragraph{Choosing the Intervention Strength.}
A larger $\alpha$ corrects more errors, but it also risks overwriting answers that were already
correct, so we choose the largest value whose collateral cost stays small. Let $\mathcal{C}$ be the
examples for color, material, and shape whose open-ended baseline answer already matches the
ground truth. Re-asking the same question under the intervention, the \emph{damage rate} is the
fraction of these that become incorrect:
\[
\mathrm{DamageRate}(\alpha)
=\frac{\bigl|\{x\in\mathcal{C}\;:\;f_\alpha(x)\text{ incorrect}\}\bigr|}{|\mathcal{C}|}.
\]
Averaged over the three attributes, it is $3.26\%$ at $\alpha{=}1$ ($e^{\alpha}\!\approx\!2.7$),
$4.74\%$ at $\alpha{=}2$ ($\approx\!7.4$), and $9.99\%$ at $\alpha{=}3$ ($\approx\!20$). We use
$\alpha{=}2$, the largest strength that keeps the damage rate below $5\%$.

\begin{figure}[t]
    \centering
    \includegraphics[width=0.98\linewidth]{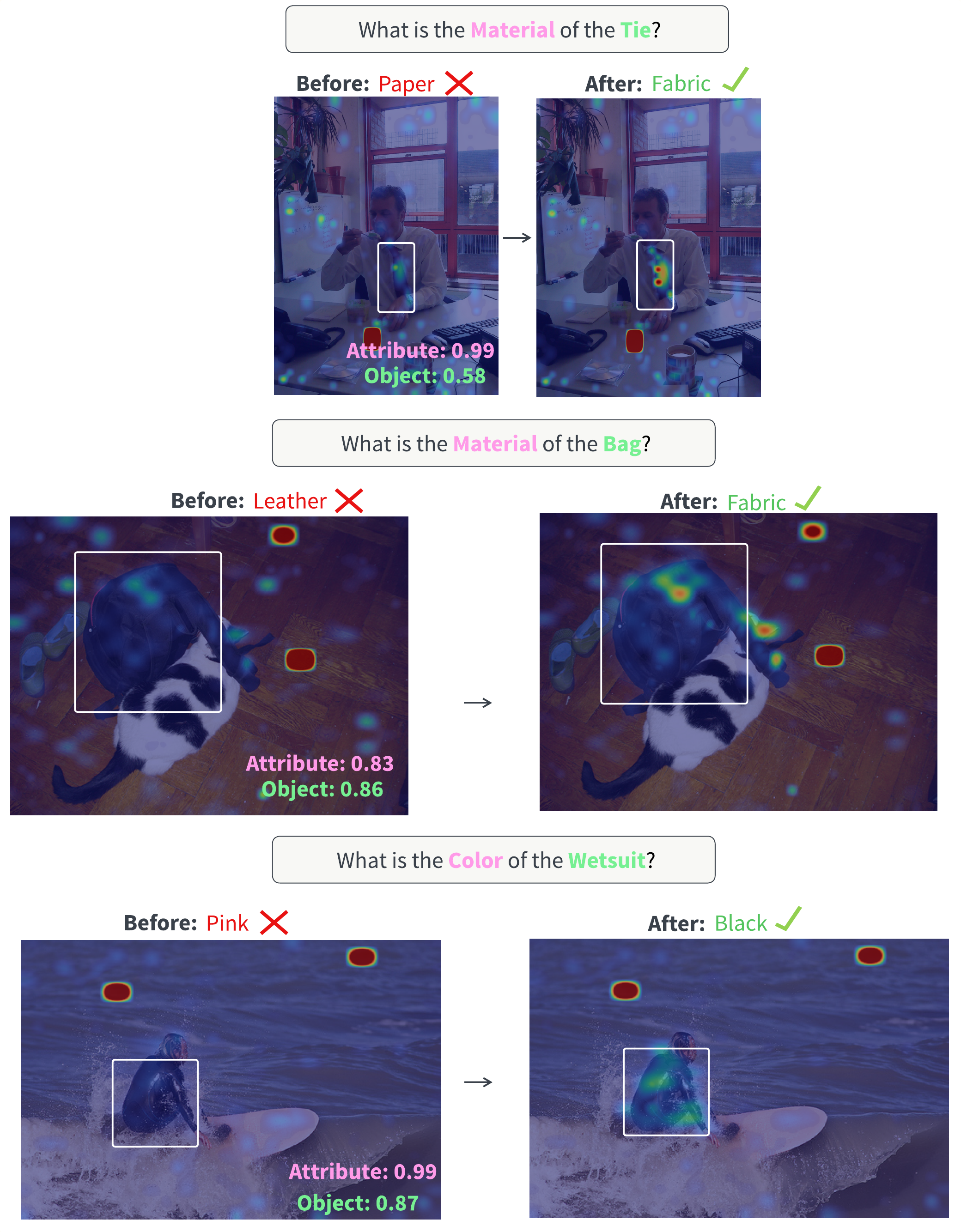}
    \setlength{\abovecaptionskip}{2pt}
    \caption{\textbf{Additional examples of the attention intervention.} Extends
    Figure~\ref{fig:attention_redirection_results}. Three evidence-positive
    errors, where \method detects both object and attribute evidence (scores shown per panel), and
    increasing attention to the annotated target-object region corrects the answer. Each panel shows
    attention from the \emph{last prompt token} to image patches, averaged over the
    last four decoder layers (head-mean). Left: no intervention. Right: prompt$\rightarrow$
    inside-object visual-token attention increased at strength $\alpha{=}2$, chosen as described in
    \S\ref{app_sub:attention-boost}.}
    \label{fig:app_attention_redirection}
\end{figure}

\section{\method\ Beyond the Evaluated Question Set}
\label{app:extensions}

Our experiments focus on a controlled class of visual questions: attribute questions about a known target object, with color, material, or shape as the queried category. This setting allows visual-evidence detection to be evaluated precisely, as the evidence required for answering is known and spatially grounded in an annotated object region $b$. The formulation of Section~\ref{subsec:problem_formulation} is not restricted to this class. It asks whether the VLM representations contain sufficient evidence to support the ground-truth answer, while \method operationalizes that evidence through three inputs: a target token set $\mathcal T^+$, a competing token set $\mathcal T^-$, and the visual tokens $\mathcal K$ over which evidence is measured. For object-attribute questions, this evidence is further decomposed into object and attribute components (Section~\ref{subsec:visual_evidence_object_attribute}). We characterize here which broader classes of questions admit analogous decompositions, and which require further extensions of the probe.

\subsection{A Taxonomy of Visual Questions}

Visual questions can be categorized along several dimensions. Linguistically, they may be polar questions (``\emph{Is the car yellow?}'') or wh-questions requesting an object, attribute, location, or number (``\emph{What color is the car?}'', ``\emph{Where is the dog?}'', ``\emph{How many cars are there?}''). Semantically, they may concern object existence or identity, one or several attributes, relations between objects, counting, or information requiring external knowledge. Most importantly for \method, they differ in the structure of the evidence required for answering them: a single concept in a single region, several concepts associated with the same region, several regions and the relations between them, or global scene-level information.

This last distinction determines whether \method applies directly. Let $C(Q,Y)=\{c_1,\ldots,c_m\}$ denote the visual concepts that must hold for $Y$ to be the answer to $Q$, each component $c_j$ measured with its own inputs $(\mathcal T^+_j,\mathcal T^-_j,\mathcal K_j)$. The composite probe then generalizes to
\begin{equation*}
M(X,Y)=\bigwedge_{j=1}^{m}M_j(X,Y),
\end{equation*}
where each $M_j$ tests whether the representations contain the corresponding required evidence, and the object-attribute probe of Section~\ref{subsec:visual_evidence_object_attribute} is the case $C=\{\mathrm{object},\mathrm{attribute}\}$. The central requirement is therefore not a particular question form, polar or wh-, but whether the evidence supporting an answer decomposes into identifiable grounded components.

\subsection{Question Classes That Naturally Fit This Formulation}

\paragraph{Attribute questions and verification.}
Our experiments consider wh-style attribute questions such as ``\emph{What is the color of the car?}'', but can be naturally extended to verification questions such as ``\emph{Is the car yellow?}'', which require the same two evidence components, one for the target object and one for the queried value. When the queried value has a closed candidate set, all three inputs are unchanged and the composite probe applies directly. When it does not, as in ``\emph{Is the shirt striped?}'', $\mathcal T^-$ extends beyond alternatives from the same category, for example to the rest of the vocabulary as in the object probe, which changes the scale of the score and requires refitting $\mathrm{LR}_{\mathrm{attr}}$, but not the formulation itself.

\paragraph{Object presence and localization.}
Questions such as ``\emph{Is there a bicycle in the image?}'' require evidence for an object concept rather than an object-attribute pair, with $\mathcal T^+$ representing the target object and $\mathcal T^-$ competing object concepts. For presence questions, $\mathcal K$ covers the whole image, $\mathcal K_{\mathrm{in}}\cup\mathcal K_{\mathrm{out}}$, and supervision comes from image-level presence annotations rather than from the in/out contrast of Section~\ref{subsec:probe_supervision}. Localization questions such as ``\emph{Where is the yellow car?}'' combine the same two components, but retain the local scores of Eq.~\ref{eq:patch_evidence_score} instead of aggregating them over $\mathcal K$, so that contiguous high-evidence tokens can be clustered into candidate regions. The same evidence thus answers two different questions, depending on whether it is aggregated or kept per visual token.

\paragraph{Compositional single-object questions.}
The formulation also extends when several properties of one object jointly determine the answer, as in ``\emph{Is the large metal cube red?}''. Each value contributes its own $\mathcal T^+_j$ and category-specific $\mathcal T^-_j$, while all components share the object's region $\mathcal K_{\mathrm{in}}$, and the composite probe is their conjunction. Syntactic complexity alone therefore poses no difficulty: questions with different surface realizations are handled identically when they reduce to the same grounded components.

\subsection{Question Classes Requiring More Substantial Extensions}

The setting becomes more challenging when the answer cannot be determined from evidence associated with individual concepts within a single object region. \textbf{Relational and comparative questions}, such as ``\emph{Is the cat to the left of the dog?}'' or ``\emph{Which car is larger?}'', require evidence not only for both entities but for the relation between them, and detecting the cat and the dog separately does not establish that one is left of the other. Extending \method would therefore require components defined over tuples of regions, replacing $M_j(X,Y;\mathcal K_j)$ with evidence functions over $\mathcal K_1,\ldots,\mathcal K_r$.

\textbf{Counting and global scene questions} present a related but distinct challenge. ``\emph{How many cars are there?}'' requires individuating instances and aggregating them without counting the same evidence twice, whereas ``\emph{Is this a crowded street?}'' depends on distributed scene-level evidence. Answering either from the local scores of Eq.~\ref{eq:patch_evidence_score} would require grouping them into instances and counting or summarizing the resulting groups, rather than testing a fixed set of components for their presence.

Finally, \textbf{knowledge-dependent questions}, such as ``\emph{What company manufactures this car?}'', require visual evidence together with knowledge stored parametrically in the language model. The distinction between visual-representation and visual-utilization failures is then less identifiable, as an incorrect answer may reflect missing visual evidence, missing knowledge, or a failure to combine the two. In such cases $\phi_X$ no longer isolates visual information, and $1-\mu_\phi$ no longer measures a visual-representation failure rate.

Overall, our experiments establish \method in a deliberately controlled setting rather than claiming empirical coverage of the full space of visual questions. Nevertheless, the method is not inherently tied to color, material, and shape questions. Its natural scope is broader: questions for which the evidence supporting an answer can be represented as a structured composition of visually grounded concepts. Extending from single-region components to relations, sets of instances, and mixed visual and parametric evidence offers a natural path toward diagnosing representation and utilization failures in more general multimodal tasks.

\end{document}